\documentclass[journal]{IEEEtran}

\usepackage{cite}
\usepackage{amsmath,amssymb,amsfonts}
\usepackage{algorithm}
\usepackage{algorithmic}
\usepackage{graphicx}
\usepackage{textcomp}
\usepackage{color}
\usepackage[cmintegrals]{newtxmath}
\usepackage[caption=false]{subfig}

\newtheorem{proposition}{Proposition}
\newtheorem{definition}{Definition}
\newtheorem{theorem}{Theorem}

\begin{document}
	
	\title{Bandwidth Allocation with Device Partitioning for Federated Learning over Industrial IoT networks}
	
	\author{Kangmin Kim and Jaeyoung Song~\IEEEmembership{Member,~IEEE,}
		% <-this % stops a space
		\thanks{K. Kim was with the School of Electrical and Electronics Engineering, Pusan National University, Busan, 46241, Korea (e-mail : km.kim@pusan.ac.kr).}
		\thanks{J. Song is with the School of Electrical and Electronics Engineering, Pusan National University, Busan, 46241, Korea (e-mail : jsong@pusan.ac.kr).}% <-this % stops a space
	}
	\maketitle

	% The paper headers
	%\markboth{\journalname, VOL. XX, NO. XX, XXXX}
	%{Author \MakeLowercase{\textit{et al.}}: Title}

	\begin{abstract}
		We consider a federated learning (FL) system in which Industrial 
		Internet-of-Things (IIoT) devices collaboratively train a global model over 
		wireless channels without sharing local data. In such systems, communication 
		time is a primary bottleneck that constrains overall training efficiency. 
		Unlike conventional networks that prioritize individual quality-of-service 
		requirements, FL systems collectively aim to converge to an optimal global 
		model as efficiently as possible, which calls for a fundamentally different 
		approach to bandwidth allocation. In this paper, we propose a novel bandwidth 
		allocation policy that exploits the heterogeneity of device computing 
		capabilities to minimize total training time. Rather than distributing 
		bandwidth among all selected devices simultaneously, the proposed policy 
		partitions the participating devices into ordered subsets and sequentially 
		grants each subset exclusive access to the full bandwidth. We formally prove that this partitioning-based policy achieves a strictly lower training time than any bandwidth allocation scheme without partitioning, irrespective of the underlying scheduling algorithm. Furthermore, by reducing per-device 
		transmission duration, the proposed policy also minimizes uplink energy 
		consumption, which is particularly beneficial for battery-constrained IIoT 
		devices. Extensive experiments on real-world datasets — including GC10-Det, 
		an industrial surface defect benchmark, and CIFAR-10, a standard image 
		classification benchmark — demonstrate that the proposed policy consistently 
		reduces training time and energy consumption compared to existing bandwidth 
		allocation schemes, approaching the theoretical lower bound on round time.
	\end{abstract}
	
	
	\section{Introduction}
		The rapid evolution of Industry 4.0 and the proliferation of Industrial 
        Internet-of-Things (IIoT) devices have generated an unprecedented volume of 
        operational data. Leveraging artificial intelligence (AI), this data has become 
        a critical resource for optimizing manufacturing processes, enabling predictive 
        maintenance, and enhancing quality control~\cite{sodhro2019tii}. Nevertheless, 
        the sheer scale of locally generated data, combined with stringent requirements 
        for protecting proprietary industrial information, poses fundamental challenges 
        to conventional machine learning approaches that rely on centralizing raw data 
        for model training.

        Federated learning (FL) has emerged as a promising paradigm to address these 
        challenges in industrial settings~\cite{nguyen2021wc}. Rather than transferring 
        raw data to a central server, FL enables IIoT devices to perform local model 
        training and transmit only model updates. This collaborative framework allows 
        multiple devices to jointly train a shared global model while preserving the 
        data privacy of individual manufacturing units.

        Despite its privacy-preserving merits, FL inherently requires frequent 
        communication between participating devices and the central server. In 
        industrial environments, which often span complex, multi-domain networks with 
        severely constrained wireless resources, the repeated transmission of model 
        updates constitutes a major deployment bottleneck~\cite{lu2021tii}. A 
        substantial body of research has addressed this challenge from an algorithmic 
        perspective, proposing strategies such as model compression~\cite{shah2021model},
        reduced aggregation frequency~\cite{luping2019cmfl}, local update 
        momentum~\cite{liu2020accelerating}, and adaptive aggregation weight 
        allocation~\cite{wu2021fast}.

        Complementary to algorithmic approaches, communication efficiency can be 
        fundamentally improved through the judicious allocation of physical resources 
        such as bandwidth and power. While extensive resource allocation strategies have 
        been developed for conventional communication 
        systems~\cite{kivanc2003computationally, 4723350, liu2019resource}, these are ill-suited for FL due to a fundamental difference in 
        objective. Traditional networks optimize for stable, high data rates to 
        facilitate rapid individual message delivery. In contrast, FL systems do not 
        require each device to communicate as quickly as possible; rather, their 
        collective goal is to converge to an optimal global model rapidly and efficiently. This 
        distinction motivates the need for FL-specific resource management strategies 
        tailored to industrial settings.

        Accordingly, a growing body of research has investigated communication resource 
        allocation over wireless channels to enhance FL efficiency. In~\cite{gao2021tii}, 
        a resource allocation framework was proposed to minimize training latency in IIoT 
        environments. To mitigate delays caused by stragglers, the authors introduced a 
        greedy client sampling algorithm that considers both channel conditions and data 
        quality. Similarly, the work in~\cite{xu2020client} jointly optimized bandwidth 
        allocation and device energy budgets to minimize overall consumption. Recognizing 
        that IIoT devices often operate under stringent energy constraints, the authors 
        proposed a strategy that selects clients and allocates bandwidth based on real-time 
        battery levels and channel quality. 

        In~\cite{ko2021joint}, the joint client selection and bandwidth allocation 
        problem was formulated where clients providing high-quality updates are prioritized.
        Furthermore, \cite{kuang2021client} proposed strategies to maximize the number 
        of devices capable of satisfying a predefined time constraint without requiring 
        prior channel state information. This was achieved through a combinatorial 
        multi-armed bandit algorithm and a bandwidth allocation strategy analogous 
        to the water-filling algorithm. Finally, \cite{zhao2021system} investigated the 
        joint optimization of client selection and bandwidth allocation under strict 
        round deadlines, exploring the fundamental tradeoff between the number of 
        participating clients and the allocated bandwidth per device.
        
        Several studies~\cite{ren2020scheduling, shi2020device, shi2020joint} pursued bandwidth 
        allocation strategies that ensure all devices complete communication 
        simultaneously. In particular, \cite{ren2020scheduling} assumed synchronous initiation; thus, bandwidth is allocated to have the same transmission rate across selected devices.
        In addition, \cite{shi2020device} relaxed this to asynchronous initiation and 
        established optimality when transmission power scales linearly with 
        bandwidth. Extending to fixed transmission power, \cite{shi2020joint} demonstrated optimality of bandwidth allocation which ensures the simultaneous completion of transmission. 

        Several studies~\cite{ren2020scheduling, shi2020device, shi2020joint} have pursued 
        bandwidth allocation strategies aimed at ensuring all devices complete their 
        transmission simultaneously. Specifically, the authors of~\cite{ren2020scheduling} 
        assumed synchronous initiation; consequently, bandwidth is allocated to maintain 
        uniform transmission rates across all selected devices. This approach was further 
        extended in~\cite{shi2020device}, which relaxed the assumption to asynchronous 
        initiation and established optimality under the condition that transmission power 
        scales linearly with bandwidth. Addressing scenarios with fixed power budgets where 
        power does not scale with bandwidth, \cite{shi2020joint} demonstrated that 
        optimality is still achieved through bandwidth allocation strategies that 
        enforce simultaneous transmission completion.

        However, all of the above schemes~\cite{xu2020client, ren2020scheduling, 
        ko2021joint, shi2020device, shi2020joint, kuang2021client, zhao2021system} share 
        two critical limitations that restrict their applicability to realistic IIoT 
        deployments. First, they assume that all selected devices share the entire 
        bandwidth simultaneously. In practice, IIoT environments involve heterogeneous 
        devices — such as robotic arms, programmable logic controllers, and smart sensors — with widely varying local training times. Under simultaneous allocation, 
        devices that finish training early must wait idly, incurring unnecessary time 
        overhead. Alternatively, sequentially allocating the full bandwidth to subsets 
        of devices can substantially reduce total communication time by exploiting this 
        heterogeneity. Second, these works assume a fixed SNR per device, independent of 
        bandwidth allocation. Under realistic fixed transmission power constraints, 
        which arise from hardware limitations or safety regulations in industrial 
        settings~\cite{zhang2023drl}, noise power varies with allocated bandwidth, 
        rendering SNR a function of bandwidth and significantly complicating the 
        optimization problem.

        To address these gaps, we propose a novel bandwidth allocation policy for 
        IIoT-based FL that introduces device partitioning into the communication 
        schedule. Our approach explicitly accounts for fixed transmission power 
        constraints, under which SNR is treated as a function of bandwidth, faithfully 
        reflecting practical industrial conditions. The proposed policy is agnostic to 
        the underlying FL algorithm, making it compatible with methods addressing data 
        heterogeneity~\cite{zhao2018federated, li2020federatedhetero}, model 
        compression~\cite{sattler2019robust, shah2021model, xu2020ternary}, and 
        federated dropout~\cite{caldas2018expanding, wen2022federated}. Furthermore, we 
        formally prove that our partitioning-based policy outperforms the optimal 
        simultaneous bandwidth allocation strategy, regardless of the client scheduling 
        algorithm employed.

        The main contributions of this paper are summarized as follows:

		\begin{itemize}
			\item We investigate a practical bandwidth allocation problem for 
			industrial FL under fixed device transmission power, faithfully capturing 
			hardware constraints prevalent in IIoT deployments.

			\item We propose a novel device partitioning policy in which heterogeneous 
			devices are divided into sequential groups, each granted exclusive access 
			to the full bandwidth. We formally prove that this policy outperforms any 
			simultaneous allocation scheme, irrespective of the client scheduling 
			algorithm.

			\item We develop a low-complexity partitioning algorithm that achieves 
			near-optimal performance while avoiding the prohibitive computational cost 
			of exhaustive search.

			\item We show that the proposed policy reduces energy consumption 
			of IIoT devices as a byproduct of shorter transmission durations, which 
			is particularly beneficial for energy-constrained devices in industrial 
			deployments.

			\item We conduct extensive experiments on real-world datasets, 
			demonstrating significant reductions in FL round completion time and 
			energy consumption across diverse bandwidth, computation load, and 
			participation rate configurations.
		\end{itemize}

        The remainder of this paper is organized as follows. 
		Section~\ref{sec:system_model} describes the system model. 
		Section~\ref{sec:problem_formulation} presents the bandwidth allocation problem 
		formulation. Section~\ref{sec:bw_partition} details the proposed policy. 
		Section~\ref{sec:exp} reports experimental results on real-world datasets. 
		Finally, Section~\ref{sec:conclusion} concludes the paper.

	\section{System Model}\label{sec:system_model}		
		\subsection{Federated Learning}
		\begin{figure}
			\centering
			\includegraphics[width=\columnwidth]{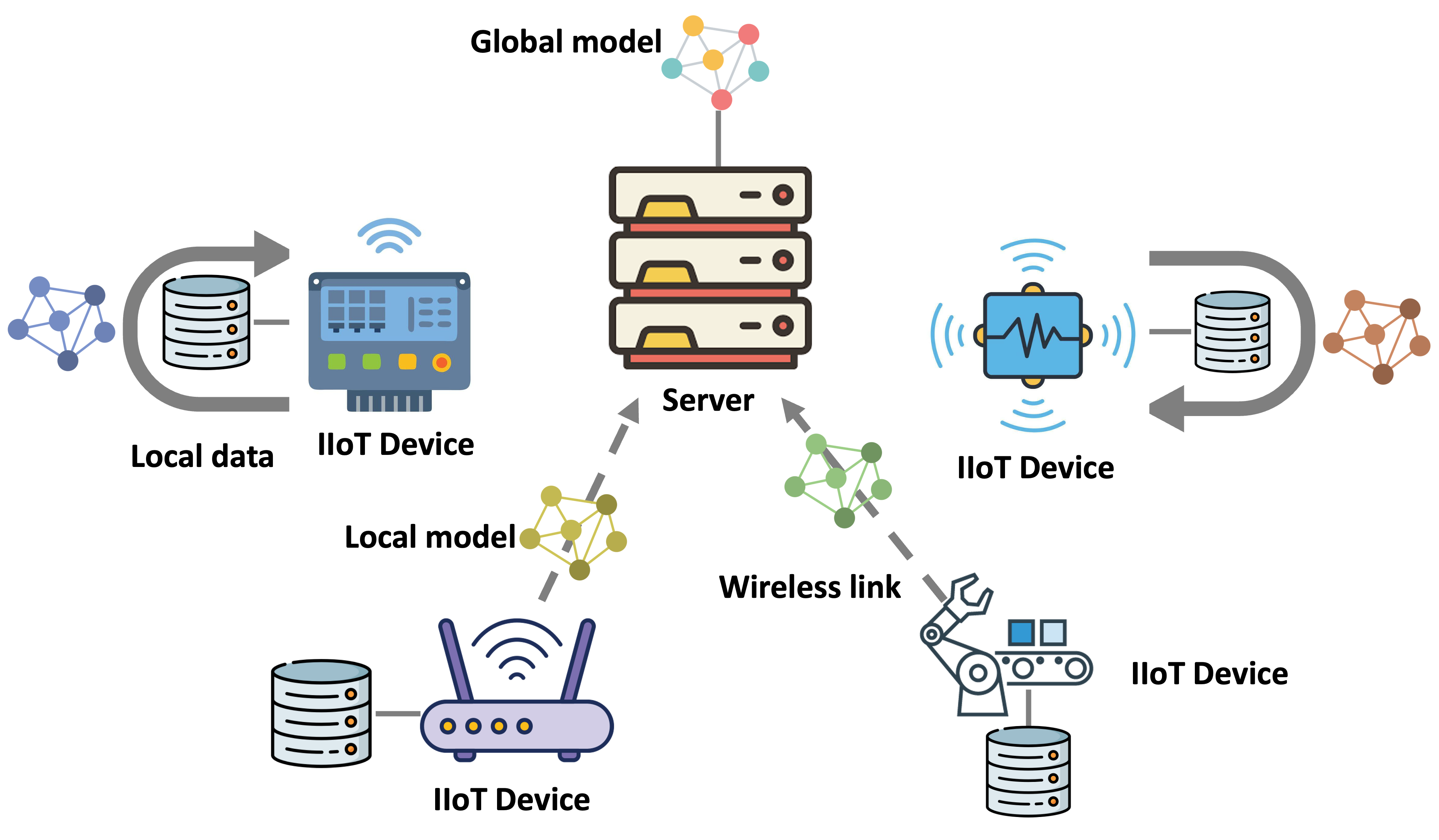}
			\caption{Federated learning system over IIoT networks}
			\label{fig:system_model}
		\end{figure}

		We consider an FL system deployed over an IIoT network consisting of a central 
		server and $N$ heterogeneous IIoT devices, denoted by the set 
		$\mathcal{N} = \{1, \dots, N\}$ as illustrated in Fig.~\ref{fig:system_model}. Each device $n \in \mathcal{N}$ holds a private 
		local dataset $\mathcal{D}_n$ of size $D_n = |\mathcal{D}_n|$, and the global 
		dataset is given by $\mathcal{D} = \bigcup_{n=1}^N \mathcal{D}_n$.

		The objective of FL is to collaboratively train a global model that minimizes 
		the average loss over $\mathcal{D}$ without exposing the local data of any 
		device. Given a loss function $f(\mathbf{w}; x)$ for model parameters $\mathbf{w}$ 
		and data sample $x$, the global loss function is defined as
		\begin{align}
			F(\mathbf{w}) = \frac{1}{|\mathcal{D}|} \sum_{x \in \mathcal{D}} 
			f(\mathbf{w}; x). \label{eq:g_loss}
		\end{align}

		\begin{figure}[t]
			\centering
			\subfloat[Local update]{\label{subfig:local_update}
				\includegraphics[width=0.45\columnwidth]{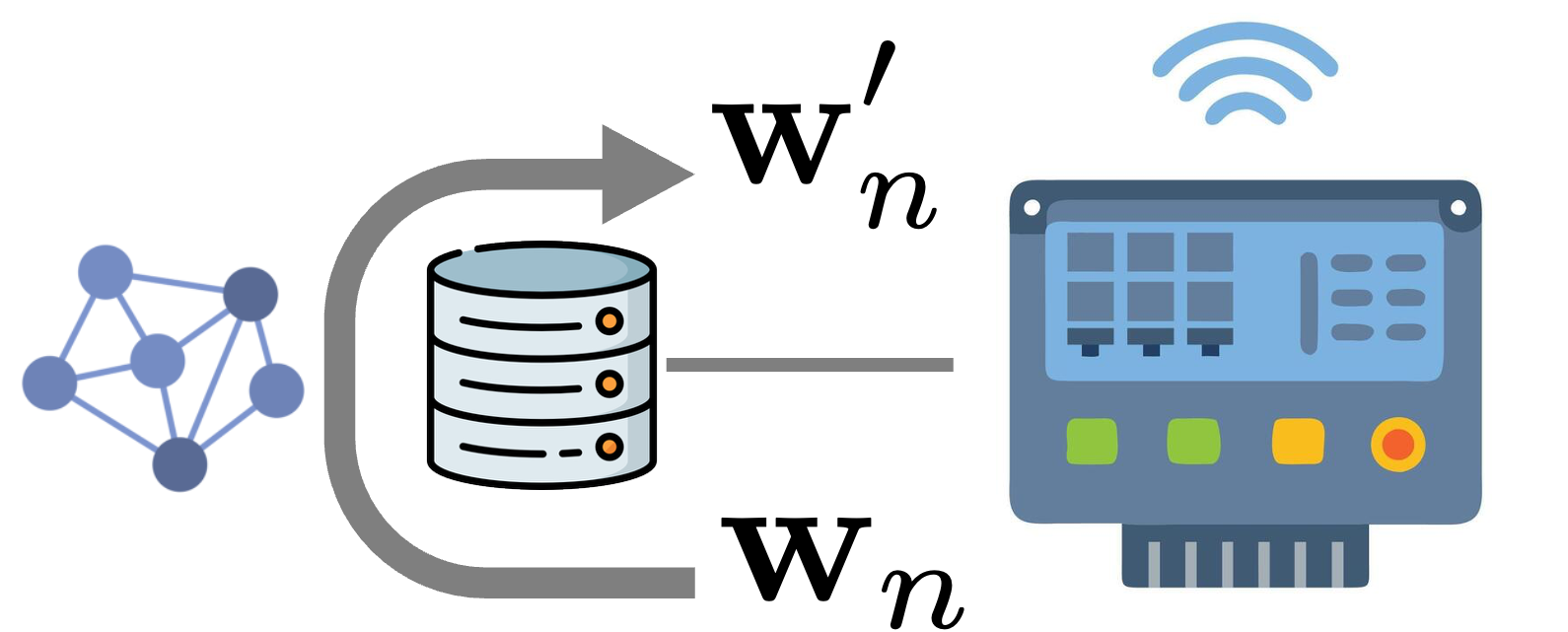}}
				\hfill
			\subfloat[Transmission of local model]{\label{subfig:local_tran}
				\includegraphics[width=0.45\columnwidth]{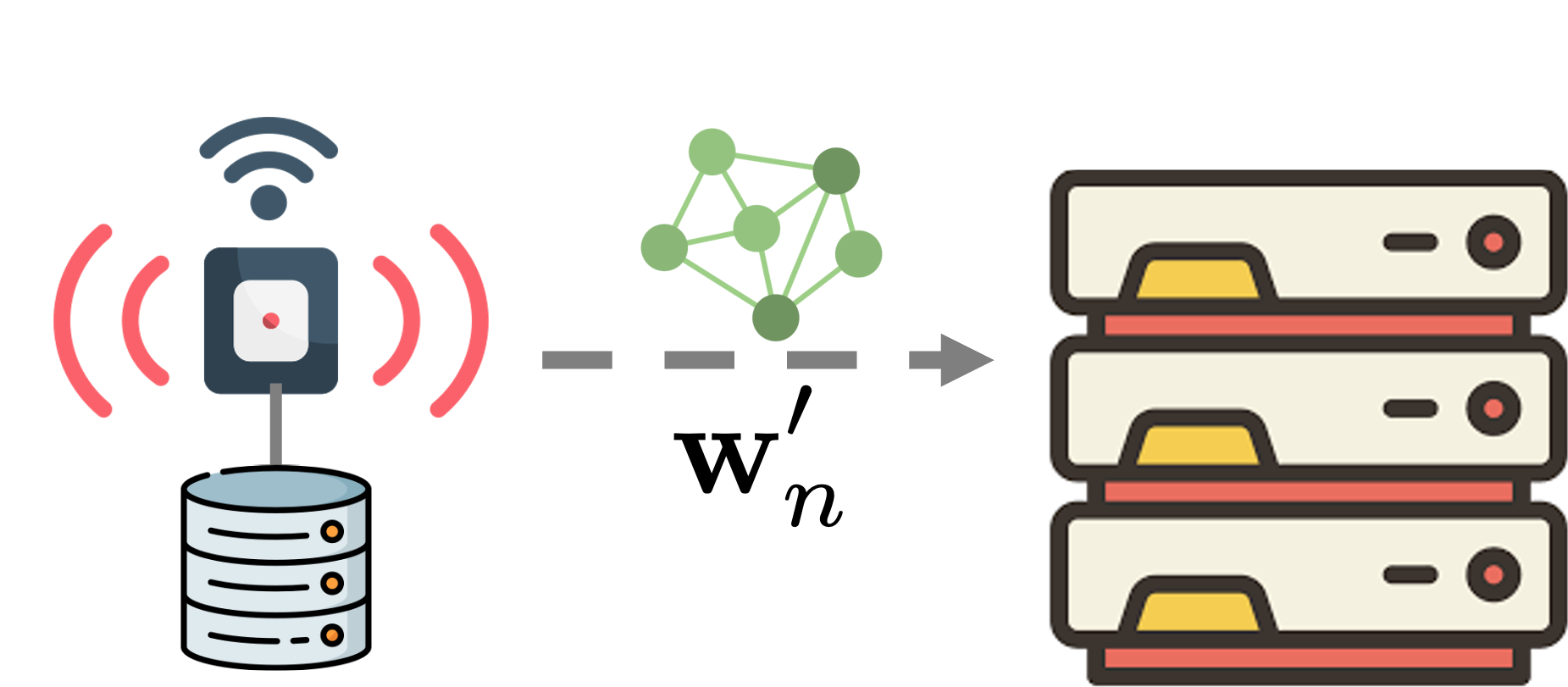}}
				
			\subfloat[Global update]{\label{subfig:global_update}
				\includegraphics[width=0.45\columnwidth]{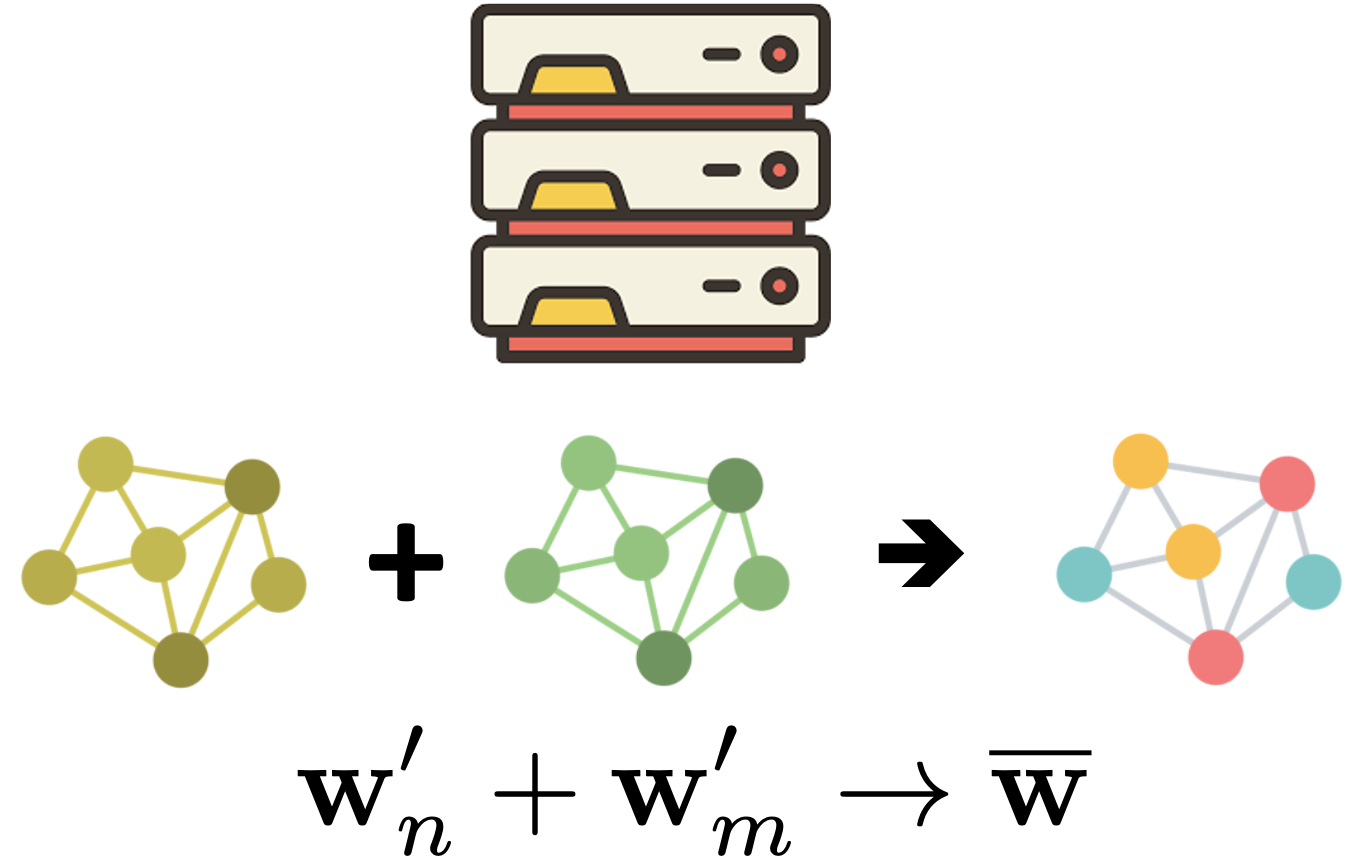}}
				\hfill
			\subfloat[Transmission of global model]{\label{subfig:global_tran}
				\includegraphics[width=0.45\columnwidth]{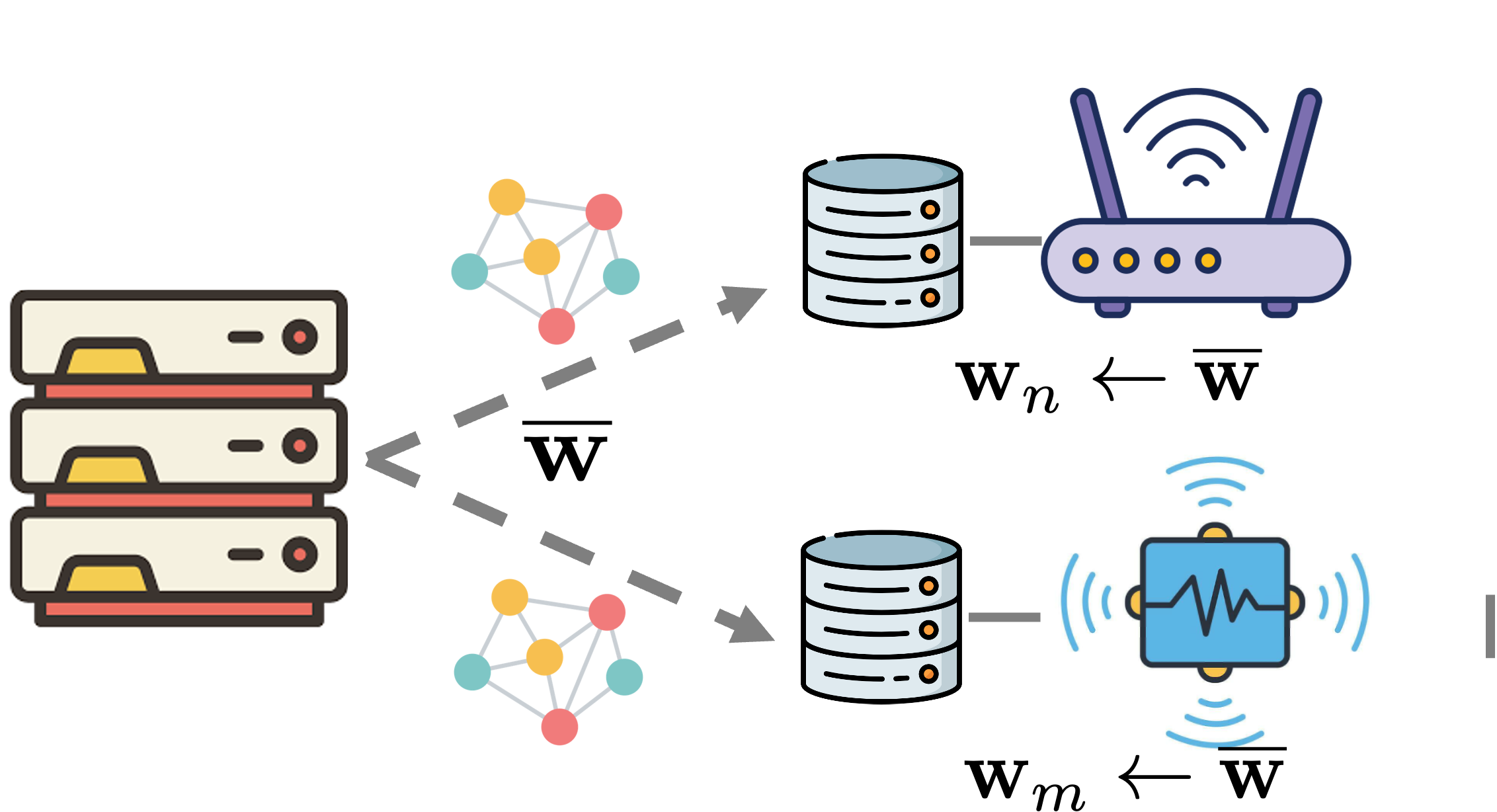}}
			\caption{Procedures of federated learning}
			\label{fig:proc_FL}
		\end{figure}	

		To minimize $F(\mathbf{w})$ while preserving data privacy, the FL system 
		proceeds iteratively over multiple communication rounds through four stages: 
		global model broadcasting, local updating, local model transmission, and global 
		aggregation, as illustrated in Fig.~\ref{fig:proc_FL}.

		At the start of round $k$, the server selects a subset of devices 
		$\mathcal{N}_k \subseteq \mathcal{N}$ with $|\mathcal{N}_k| = q N$, where 
		$q \in (0, 1]$ is the participation fraction, and broadcasts the current global 
		model $\overline{\mathbf{w}}_k$ to all selected devices. Each device 
		$n \in \mathcal{N}_k$ initializes its local model as 
		$\mathbf{w}_{k,n} = \overline{\mathbf{w}}_k$ and performs local training via 
		mini-batch stochastic gradient descent (SGD). For a mini-batch 
		$\mathcal{Q}_{n} \subset \mathcal{D}_n$ of size $Q$, the stochastic gradient 
		computed by device $n$ is
		\begin{align}
			\mathbf{g}_{k,n} = \frac{1}{Q} \sum_{x \in \mathcal{Q}_{n}} 
			\nabla f\left(\mathbf{w}_{k,n}; x\right).
		\end{align}
		Each device then updates its local model over $E$ epochs, where a single update 
		step is given by
		\begin{align}
			\mathbf{w}^{\text{next}}_{k,n} = \mathbf{w}^{\text{current}}_{k,n} 
			- \eta \mathbf{g}_{k,n},
		\end{align}
		with $\eta$ denoting the learning rate.

		Upon completing local training, each device $n$ transmits its updated local 
		model $\mathbf{w}'_{k,n}$ to the server via the uplink wireless channel. The 
		server then aggregates the received models as
		\begin{align}
			\overline{\mathbf{w}}_{k+1} = \sum_{n \in \mathcal{N}_k} 
			\frac{D_n}{\sum_{n' \in \mathcal{N}_k} D_{n'}} \mathbf{w}'_{k,n}.
		\end{align}
		This process repeats until the global model achieves a target optimality gap of 
		$\epsilon$. Based on convergence analyses in~\cite{li2019convergence, 
		cho2022towards}, we assume that $K$ communication rounds are required to reach 
		this target. 
		
		\subsection{Computing Time and Communication Time}
		In each FL round, round time comprises the local computing time 
		at the IIoT devices and communication time. Since the edge server has 
		substantial computational and communication resources, the time for global aggregation and transmission of global model is 
		considered negligible.

		\subsubsection{Computing Time}
		IIoT devices exhibit heterogeneous computing capabilities, characterized by CPU 
		cycle frequencies drawn from finite set $\mathcal{C}$. Let $c_n \in \mathcal{C}$ denote the CPU cycle frequency of device $n$. Assuming that processing one data sample requires $\rho$ CPU 
		cycles, the computing time for device $n$ to complete $E$ local epochs over its 
		dataset of size $D_n$ is
		\begin{align}
			\zeta_n = \frac{\rho E D_n}{c_n}. \label{eq:comp_time}
		\end{align}
		Without loss of generality, devices are indexed in ascending order of computing 
		time, i.e., $\zeta_n \leq \zeta_{m}$ for all $n < m$.
		
		\subsubsection{Communication Time}
		The transmission time for the server to broadcast the global model is 
		considered negligible due to the server's high transmission 
		power~\cite{yang2020energy, feng2021min}. Accordingly, the communication 
		bottleneck is dominated by the transmission of local models by IIoT devices. Let $B$ 
		denote the total available system bandwidth. When a bandwidth of $b_n \leq B$ 
		is allocated to device $n$ in round $k$, the transmission rate is
		\begin{align}
			R_{k,n}(b_n) = b_n \log_2\!\left(1 + \frac{|h_{k,n}|^2 P}{b_n N_0}
			\right), \label{eq:capacity}
		\end{align}
		where $P$ is the transmit power, $N_0$ is the noise power spectral density, and 
		$h_{k,n}$ is the channel fading coefficient. We assume that the channel remains 
		static within each round but varies across rounds~\cite{guo2022joint, 
		yao2024wireless}. Given that the size of local model is equal to $W$ bits, the transmission time for device $n$ in round $k$ is
		\begin{align}
			\delta_{k,n}(b_n) = \frac{W}{b_n \log_2\!\left(1 + 
			\dfrac{|h_{k,n}|^2 P}{b_n N_0}\right)}. \label{eq:comm_time}
		\end{align}

        \section{Problem Formulation}\label{sec:problem_formulation}
        In each round, IIoT devices can begin transmission only after completing 
		their local computations. Due to heterogeneous computing capabilities, devices 
		finish their local updates at different times. To exploit this heterogeneity and improve communication 	efficiency, we allow bandwidth to be reassigned once a device 
		completes its transmission.
		
		Under this reallocation strategy, the selected devices $\mathcal{N}_k$ can be 
		partitioned into multiple ordered subsets, where each subset is granted exclusive 
		access to the full bandwidth $B$. Once all devices in one subset complete their 
		transmissions, the next subset begins. We formalize this as follows.
		
		\begin{definition}
			A bandwidth allocation policy $\pi$ is a mapping that assigns a partition of 
			$\mathcal{N}_k$ and a corresponding bandwidth vector to each subset:
			\begin{align}
				\pi_{\mathcal{N}_k} = (\mathcal{P}, \mathcal{B}),
			\end{align}
			where $\mathcal{P} = \{\mathcal{S} \mid \mathcal{S} \subset \mathcal{N}_k\}$ is 
			an ordered partition of $\mathcal{N}_k$, and 
			$\mathcal{B} = \{\mathbf{b}_{\mathcal{S}} \in \mathbb{R}^{|\mathcal{S}|} \mid 
			\forall \mathcal{S} \in \mathcal{P}\}$ is the set of bandwidth allocation 
			vectors, each satisfying the total bandwidth constraint $B$.
		\end{definition}

		We denote by $\phi$ the class of single-partition policies, i.e., 
		$\phi_{\mathcal{N}_k} = \bigl(\{\mathcal{N}_k\}, \{\mathbf{b}_{\mathcal{N}_k}
		\}\bigr)$, in which all selected devices share the bandwidth simultaneously. 
		
		The subsets in $\mathcal{P}$ are ordered by transmission sequence. For a device 
		$n$ belonging to subset $\mathcal{S}$, its transmission cannot begin until all 
		devices in the preceding subset have completed their transmissions. The round 
		time of device $n \in \mathcal{S}$ in round $k$ is therefore given by
		\begin{align}
			\tau_{k,n}(\pi_{\mathcal{N}_k}) = \max\bigl\{\zeta_n,\, 
			\tau_{k,\mathcal{S}'}(\pi_{\mathcal{N}_k})\bigr\} + 
			\delta_{k,n}(b_n(\pi_{\mathcal{N}_k})), \label{eq:round_time_client}
		\end{align}
		where $\mathcal{S}'$ is the preceding subset of $\mathcal{S}$, and 
		$\tau_{k,\mathcal{S}'}(\pi_{\mathcal{N}_k})$ denotes the completion time of devices' transmission in subset $\mathcal{S}'$, defined as
		\begin{align}
			\tau_{k,\mathcal{S}'}(\pi_{\mathcal{N}_k}) = \max_{n' \in \mathcal{S}'} 
			\tau_{k,n'}(\pi_{\mathcal{N}_k}). \label{eq:round_time_part}
		\end{align}
		Since the round completes when the last subset finishes, the round time for 
		round $k$ is
		\begin{align}
			\tau_k(\pi_{\mathcal{N}_k}) = \max_{\mathcal{S} \in \mathcal{P}}\, 
			\tau_{k,\mathcal{S}}(\pi_{\mathcal{N}_k}). \label{eq:round_time}
		\end{align}
		Summing over all $K$ rounds, the total training time is
		\begin{align}
			\tau = \sum_{k=1}^{K} \max_{\mathcal{S} \in \mathcal{P}}\, 
			\tau_{k,\mathcal{S}}(\pi_{\mathcal{N}_k}).
		\end{align}

		The round time is lower-bounded by the time for the device with the longest 
		computation time to complete its transmission when allocated the full bandwidth 
		$B$. Formally,
		\begin{align}
			\tau_k(\pi_{\mathcal{N}_k}) \geq \zeta_{\hat{n}} + \delta_{k,\hat{n}}(B),
		\end{align}
		where $\hat{n} = \arg\max_{n \in \mathcal{N}_k} \zeta_n$. We denote this lower 
		bound as $\tau_k^{\text{LB}} = \zeta_{\hat{n}} + \delta_{k,\hat{n}}(B)$. To 
		isolate the effect of bandwidth allocation, we define the \emph{round time gap}
		\begin{align}
			\Delta_k(\pi_{\mathcal{N}_k}) = \max_{\mathcal{S} \in \mathcal{P}}\, 
			\tau_{k,\mathcal{S}}(\pi_{\mathcal{N}_k}) - \tau_k^{\text{LB}}, 
			\label{eq:gap}
		\end{align}
		and seek to minimize its average over all rounds:
		\begin{align}
			\bar{\Delta} = \frac{1}{K} \sum_{k=1}^{K} \Delta_k(\pi_{\mathcal{N}_k}). 
			\label{eq:avg_gap}
		\end{align}
		The optimal bandwidth allocation policy is obtained by solving
		\begin{align}
			\textbf{P:} \quad \min_{\pi \in \mathcal{F}}\;\bar{\Delta}, 
			\label{prob:min_bw}
		\end{align}
		where $\mathcal{F}$ denotes the set of all feasible policies.

		\section{Bandwidth Allocation Policy with Device Partitioning }\label{sec:bw_partition}
			
		Solving \eqref{prob:min_bw} requires enumerating all feasible partitions of 
		$\mathcal{N}_k$, resulting in combinatorial complexity that grows exponentially 
		with $|\mathcal{N}_k|$. To gain insight toward a tractable solution, we first 
		analyze the two-device case $|\mathcal{N}_k| = 2$.
		
		\begin{proposition}\label{prop1}
			For a given communication round $k$, let $\mathcal{N}_k = \{i, j\}$ with $\zeta_i \leq \zeta_j$. The optimal policy $\pi^*$ minimizing the round time gap is
			\begin{align}
				\pi^* = \begin{cases}
					\bigl(\{i\}, \{j\},\, [B],\, [B]\bigr) 
						& \text{if } \tau_{k,i}^* \leq \tau^{\text{th}}, \\[4pt]
					\bigl(\{i,j\},\, [B - \tilde{b}_j,\, \tilde{b}_j]\bigr) 
						& \text{otherwise},
				\end{cases} \label{eq:min_round_t}
			\end{align}
			where $\tau^*_{k,i} = \zeta_i + \delta_{k,i}(B)$ is the minimum round time 
			for device $i$ when allocated the full bandwidth, and the threshold is
			\begin{align}
				\tau^{\text{th}} = \zeta_j + \delta_{k,j}(\tilde{b}_j) - \delta_{k,j}(B),
			\end{align}
			with $\tilde{b}_j$ denoting the bandwidth allocated to device $j$ such that 
			both devices complete transmission simultaneously.
		\end{proposition} 

		\begin{IEEEproof}
			All feasible policies can be classified into two cases: those in which devices 
			$i$ and $j$ share the bandwidth simultaneously, and those in which each device 
			is allocated the bandwidth exclusively in sequence. In the latter case, 
			allocating the full bandwidth to each device individually minimizes per-device 
			communication time, since $\delta_{k,n}(b_n)$ is strictly decreasing in $b_n$.
			
			When devices $i$ and $j$ share the bandwidth simultaneously, the round time 
			is minimized when both devices complete their transmissions at the same time. 
			To see this, note that the round time equals the completion time of the 
			later-finishing device. If one device finishes earlier than the other, 
			transferring bandwidth from the faster device to the slower one reduces the 
			round time. This reallocation can be applied repeatedly until both devices 
			complete transmission simultaneously, at which point the round time is 
			minimized.

			It therefore suffices to compare the policy that ensures simultaneous 
			completion of transmission against the policy that allocates the full 
			bandwidth to each device sequentially.

			The round time gap of sequential allocation is given by
				\begin{align}
					\Delta_k(\pi_{\{i,j\}}) &= \max\{\tau^*_{k,i},\, \zeta_j\} + 
					\delta_{k,j}(B) - \tau^{\text{LB}}_k, \\
					&= \left(\tau^*_{k,i} - \zeta_j\right)^+,
				\end{align}
				where $(x)^+ = \max\{x, 0\}$ denotes the rectified linear unit function.

				The round time gap of the simultaneous allocation with equal completion times 
				is given by
				\begin{align}
					\Delta_k(\pi_{\{i,j\}}) &= \zeta_j + \delta_{k,j}(\tilde{b}_j) - 
					\tau^{\text{LB}}_k, \\
					&= \delta_{k,j}(\tilde{b}_j) - \delta_{k,j}(B).
				\end{align}

				Sequential allocation achieves a smaller or equal round time gap if and only if
				\begin{align}
					\left(\tau^*_{k,i} - \zeta_j\right)^+ \leq \delta_{k,j}(\tilde{b}_j) - 
					\delta_{k,j}(B). \label{ineq:cond_prop1}
				\end{align}
				Rewriting \eqref{ineq:cond_prop1} gives
				\begin{align}
					\tau^*_{k,i} - \delta_{k,j}(B) \leq \zeta_j + \delta_{k,j}(\tilde{b}_j),
				\end{align}
				which is equivalent to $\tau_{k,i}^* \leq \tau^{\text{th}}$, completing the 
				proof.
			\end{IEEEproof}
				
		Proposition~\ref{prop1} reveals that when device $i$ completes both computation 
		and communication quickly, sequential allocation outperforms simultaneous 
		sharing. This stands in contrast to the result of~\cite{shi2020device}, which 
		showed that simultaneous allocation with equal completion times is optimal under 
		a single partition. Proposition~\ref{prop1} demonstrates that partitioning 
		devices into multiple subsets can strictly outperform the best single-partition 
		policy. This insight is generalized to arbitrary $|\mathcal{N}_k|$ in the 
		following theorem.
		
		\begin{theorem}[Superiority of Partitioning]\label{prop2}
			If there exists a subset $\mathcal{S} \subset \mathcal{N}_k$ such that
			\begin{align}
				\tau_{k,\mathcal{S}}\!\left(\phi^*_{\mathcal{S}}\right) \leq 
				\min_{n \in \mathcal{S}^c} \zeta_n, \label{ineq:cond_prop2}
			\end{align}
			where $\mathcal{S}^c = \mathcal{N}_k \setminus \mathcal{S}$ and 
			$\phi^*_{\mathcal{S}} = \bigl(\{\mathcal{S}\}, \{(\mathbf{b}^k_{\mathcal{S}})^*
			\}\bigr)$, then
			\begin{align}
				\Delta_k\!\left(\theta^*_{\mathcal{N}_k}\right) \leq 
				\Delta_k\!\left(\phi^*_{\mathcal{N}_k}\right), \label{ineq:prop2}
			\end{align}
			where $\theta^*_{\mathcal{N}_k} = \bigl(\{\mathcal{S}, \mathcal{S}^c\},\, 
			\{(\mathbf{b}^k_{\mathcal{S}})^*, (\mathbf{b}^k_{\mathcal{S}^c})^*\}\bigr)$ 
			and $\phi^*_{\mathcal{N}_k} = \bigl(\{\mathcal{N}_k\},\, 
			\{(\mathbf{b}^k_{\mathcal{N}_k})^*\}\bigr)$.

			In other words, whenever condition~\eqref{ineq:cond_prop2} holds, partitioning 
			$\mathcal{N}_k$ into $\mathcal{S}$ and $\mathcal{S}^c$ achieves a strictly 
			lower round time gap than the optimal single-partition policy.
		\end{theorem}

		\begin{IEEEproof}
			From~\eqref{eq:round_time},
				\begin{align}
					\tau_k(\theta^*_{\mathcal{N}_k}) = \max\bigl\{\tau_{k,\mathcal{S}}
					(\theta^*_{\mathcal{N}_k}),\, \tau_{k,\mathcal{S}^c}
					(\theta^*_{\mathcal{N}_k})\bigr\}.
				\end{align}
			
			Since $\theta^*_{\mathcal{N}_k}$ is a policy which partitions $\mathcal{N}_k$ into $\mathcal{S}$ and $\mathcal{S}^c$ and allocates entire bandwidth to each part optimally, the round time for each subset is equal to that of single partition policy with optimal bandwidth allocation. In other words,
			\begin{align}
				\tau_{k, \mathcal{S}}( \theta^*_{\mathcal{N}_k}) = \tau_{k, \mathcal{S}} (\phi^*_{\mathcal{S}}) \label{eq:prop2_2}
			\end{align}

			Since $\delta_{k,n}(b^*_n(\theta^*_{\mathcal{N}_k})) > 0 $ for all $n$,
			\begin{align}
				\tau_{k, \mathcal{S}^{c}} ( \theta^*_{\mathcal{N}_k}) > \max_{n \in \mathcal{S}^c} \zeta_n
			\end{align}
			Since $\delta_{k,n}(b^*_n(\theta^*_{\mathcal{N}_k})) > 0$ for all $n$,
			\begin{align}
				\tau_{k,\mathcal{S}^c}(\theta^*_{\mathcal{N}_k}) > 
				\max_{n \in \mathcal{S}^c} \zeta_n \geq \min_{n \in \mathcal{S}^c} \zeta_n 
				\geq \tau_{k,\mathcal{S}}(\theta^*_{\mathcal{N}_k}),
			\end{align}
			where the last inequality follows from~\eqref{ineq:cond_prop2} 
			and~\eqref{eq:prop2_2}. 
			Hence, we have $\tau_{k, \mathcal{S}^{c}} ( \theta^*_{\mathcal{N}_k}) >\tau_{k, \mathcal{S}}( \theta^*_{\mathcal{N}_k})$.  
			Therefore, the round time using $\theta^*_{\mathcal{N}_k}$ becomes
			\begin{align}
				\tau_k( \theta^*_{\mathcal{N}_k}) = \tau_{k, \mathcal{S}^{c}} ( \theta^*_{\mathcal{N}_k}). \label{eq:prop2_3}
			\end{align}

			Furthermore, using \eqref{eq:round_time_client} and $\underset{n \in \mathcal{S}^c}{\min} \zeta_n \geq \tau_{k, \mathcal{S}}( \theta^*_{\mathcal{N}_k})$ , we can rewrite \eqref{eq:prop2_3} as 
			\begin{align}
				\tau_k( \theta^*_{\mathcal{N}_k}) & = \max_{n \in \mathcal{S}^{c} } \left\lbrace \zeta_n + \delta_{k,n} ( b^*_n(\theta^*_{\mathcal{N}_k}) ) \right\rbrace. \label{eq:prop2_1}
			\end{align}
			
			For the single-partition policy, the round time using $\phi^*_{\mathcal{N}_k}$ can be written as
			\begin{align}
				\tau_k(\phi^*_{\mathcal{N}_k}) = \max_{n \in \mathcal{N}_k} \left\lbrace \zeta_n + \delta_{k,n}(b^*_n(\phi^*_{\mathcal{N}_k})) \right\rbrace .
			\end{align}
			Let $\tilde{n} = \underset{n \in \mathcal{N}_k}{\arg\max} \left\lbrace \zeta_n + \delta_{k,n}(b^*_n(\phi^*_{\mathcal{N}_k})) \right\rbrace$. Then, consider the following bandwidth allocation scheme for $\mathcal{S}^c$.
			\begin{align}
				\tilde{b}_n = \begin{cases}
					b^*_n(\phi^*_{\mathcal{N}_k}) & \text{ for } n \neq \tilde{n} , \\
					B - \sum_{n \neq \tilde{n}} b^*_n(\phi^*_{\mathcal{N}_k}) & \text{ for } n = \tilde{n}
				\end{cases}
			\end{align}
			Since the same amount of bandwidth is allocated for $n \neq \tilde{n}$, $\delta_{k, n}(\tilde{b}_{n}) = \delta_{k, n}(b^*_{n}(\phi^*_{\mathcal{N}_k}))$. Moreover, as $|\mathcal{S}^c| < |\mathcal{N}_k|$, $\tilde{b}_{\tilde{n}} > b^*_{\tilde{n}}(\phi^*_{\mathcal{N}_k})$, which leads to $\delta_{k, \tilde{n}}(\tilde{b}_{\tilde{n}}) < \delta_{k, \tilde{n}}(b^*_{\tilde{n}}(\phi^*_{\mathcal{N}_k}))$.
			
			Consequently, we have
			\begin{align}
				\tau_k(\phi^*_{\mathcal{N}_k}) \geq \max_{n \in \mathcal{S}^c}\left[ \zeta_n + \delta_{k,n}(\tilde{b}_n) \right]. \label{ineq:prop2_1}
			\end{align}
			Since $\theta^*_{\mathcal{N}_k}$ allocates bandwidth optimally to $\mathcal{S}^c$, we have
			\begin{align}
				\tau_k(\theta^*_{\mathcal{N}_k}) \leq \max_{n \in \mathcal{S}^c}\left[ \zeta_n + \delta_{k,n}(\tilde{b}_n) \right]. \label{ineq:prop2_2}
			\end{align}
			Combining \eqref{ineq:prop2_1} and \eqref{ineq:prop2_2} and subtracting both sides by $\tau^{\text{LB}}_k$ yields
			\begin{align}
				\tau_k(\theta^*_{\mathcal{N}_k}) - \tau_k^{\text{LB}} \leq \tau_k(\phi^*_{\mathcal{N}_k})- \tau_k^{\text{LB}},
			\end{align} which is equivalent to \eqref{ineq:prop2}.
		\end{IEEEproof}

		\begin{algorithm}[pt]
			\caption{Device Partitioning Bandwidth Policy} \label{alg:alg2}
			\begin{algorithmic}[1]
				\STATE \textbf{Input}: $\mathcal{N}_k, \mathcal{P}, \mathcal{B}$
				\STATE \textbf{Initialization}: $\mathcal{S} \leftarrow \emptyset$
				\FOR {$n\in \mathcal{N}_k $}
					\STATE $\mathcal{S} \leftarrow \mathcal{S} \cup \{ n \}$
					\STATE $\phi^*_{\mathcal{S}} = ( \{ \mathcal{S} \}, \{ \tilde{\mathbf{b}}^k_{\mathcal{S}}\})$
					\STATE Compute $\tau_{k,\mathcal{S}} ( \phi_{\mathcal{S}}^*) $
					\IF{$\tau_{k,\mathcal{S}} ( \phi_{\mathcal{S}}^*) \leq \zeta_{n+1}$}
						\STATE $\mathcal{P} \leftarrow \mathcal{P} \cup \mathcal{S}$
						\STATE $\mathcal{B} \leftarrow \mathcal{B} \cup \{ \tilde{\mathbf{b}}^k_{\mathcal{S}}\}$
						\IF{$ \mathcal{N}_k \backslash \mathcal{S} \neq \emptyset$}
							\STATE DPBP $( \mathcal{N}_k \backslash \mathcal{S}, \mathcal{P}, \mathcal{B})$ \\ \COMMENT{Recursive call of Algorithm \ref{alg:alg2}}
						\ENDIF
					\ENDIF
				\ENDFOR
				\IF {$\mathcal{S} \neq \emptyset$}
					\STATE $\mathcal{P} \leftarrow \mathcal{P} \cup \mathcal{S}$
					\STATE $\mathcal{B} \leftarrow \mathcal{B} \cup \{ \tilde{\mathbf{b}}^k_{\mathcal{S}}\}$
				\ENDIF
				\STATE \textbf{Output}: $\pi_{\mathcal{N}_k}= ( \mathcal{P} , \mathcal{B} )$
			\end{algorithmic}
		\end{algorithm}
		
        The key insight of Theorem~\ref{prop2} is that whenever a subset $\mathcal{S}$ 
		can complete all transmissions before the remaining devices finish their local 
		computations, partitioning is guaranteed to reduce the round time gap relative 
		to any single-partition policy with optimal bandwidth allocation.
        
        Since the feasible set of \eqref{prob:min_bw} grows exponentially with 
		$|\mathcal{N}_k|$, exhaustive search is computationally intractable in practice. 
		Nevertheless, condition~\eqref{ineq:cond_prop2} in Theorem~\ref{prop2} provides 
		a verifiable criterion for when partitioning is beneficial: it is always 
		advantageous when devices exhibit large gaps in their computing times.

		Building on this insight, we propose the \emph{Device Partitioning Bandwidth 
		Policy} (DPBP), whose pseudocode is given in Algorithm~\ref{alg:alg2}. DPBP 
		applies partitioning only when condition~\eqref{ineq:cond_prop2} is satisfied. 
		Starting from an empty subset $\mathcal{S}$, devices are added sequentially in 
		ascending order of computing time. After each addition, the condition is checked: 
		if satisfied, $\mathcal{S}$ is committed as one partition and DPBP is applied 
		recursively to the remaining devices $\mathcal{S}^c$; if not, the next device 
		is added and the check is repeated. If no valid subset is found after all devices 
		are considered, the remaining devices form the final subset. In the degenerate 
		case where this occurs at the initial call, all devices are placed in a single 
		partition, recovering the standard simultaneous allocation.

		\begin{figure}[pt]
			\centering
			\includegraphics[width=\columnwidth]{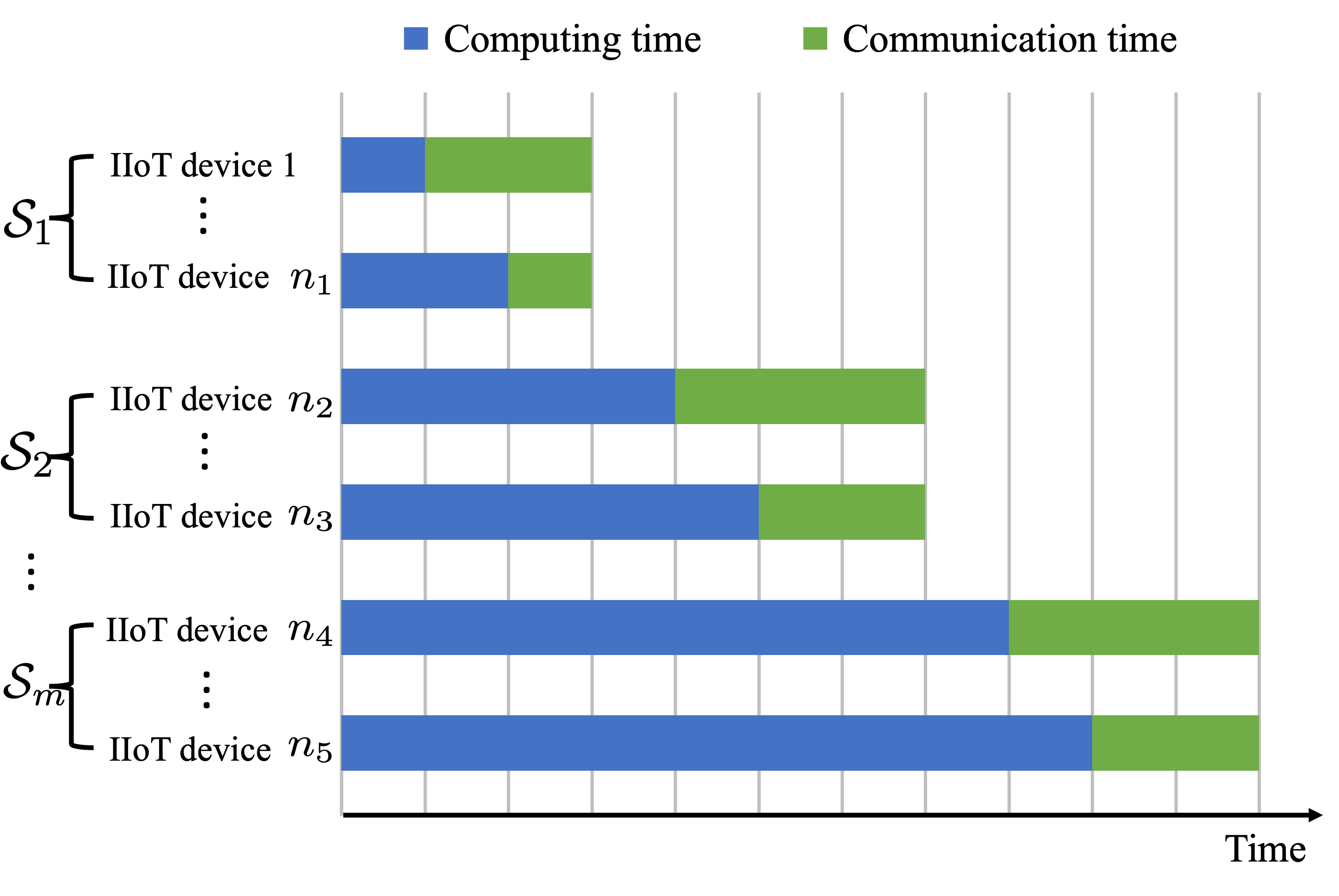}
			\caption{Timeline illustration of DPBP}\label{fig:example_DPBP}
		\end{figure}

		For each subset $\mathcal{S} \in \mathcal{P}$, the intra-subset bandwidth 
		allocation $\overline{\mathbf{b}^k}_{\mathcal{S}}$ is chosen so that all devices 
		in $\mathcal{S}$ complete transmission simultaneously using entire bandwidth, as illustrated in Fig.~\ref{fig:example_DPBP}. This allocation strategy was originally shown to be optimal under a single partition in~\cite{shi2020device}; however, 
		since~\cite{shi2020device} assumes transmission power proportional to allocated 
		bandwidth, we modify the algorithm to accommodate the fixed transmission power 
		model considered here.

		The complexity of determining the bandwidth allocation for each subset scales 
		linearly with $|\mathcal{N}_k| = qN$. The number of subsets that DPBP considers is proportional to $N$, its overall complexity is $\mathcal{O}(N^2)$. As the partitioning and 
		bandwidth allocation are performed at the server, this computational burden 
		remains manageable given the server's processing 
		capabilities~\cite{gong2014base}.
        
	    \section{Experiments}\label{sec:exp}		
		
		\subsection{Setting}
		We use two widely used real-world dataset to evaluate the performance of our proposed bandwidth allocation policy for FL. The first dataset is GC10-Det~\cite{lv2020deep}, which is a collection of 2,300 metallic surface images. Each image show different 10 defects on metallic surfaces such as silk spot, welding line, and other scratches. By training through this images, models can learn to detect various types of defects on metallic surfaces, which is crucial for quality control in manufacturing processes. The second dataset is CIFAR10~\cite{krizhevsky2009learning}, which is a widely used dataset for image classification tasks. It consists of 60,000 color images of size \( 32 \times 32 \) pixels, categorized into 10 classes, including airplanes, cars, birds, and cats. The GC10-Det dataset is used to evaluate the performance of our proposed policy in a real-world industrial application scenario, while the CIFAR-10 dataset serves as a benchmark for general federated learning scenario.

        \subsubsection{GC10-Det dataset}
        We consider $N = 20$ IIoT devices with i.i.d. local data 
        and $K = 300$ communication rounds. In each round, all devices participate. (i.e., $q=1$) The CPU cycle frequency of each device 
        $c_n$ is drawn uniformly at random from $\mathcal{C} = \{1, 5, 10, 20\} \times 10^6$ 
        cycles/s. Each device trains a CNN with eight convolutional layers and two 
        fully-connected layers, with model size $W = 20$~MB. We set $\rho = 10{,}000$ 
        CPU cycles per sample and $E = 5$ local epochs. The total bandwidth is 
        $B = 30$~MHz. The transmit power and noise power spectral density are set $\frac{P}{N_0} = 20$~dB/Mhz. The wireless channels follow Rayleigh fading.

        \subsubsection{CIFAR-10 dataset} We consider $N = 100$ IIoT devices with heterogeneous local 
        data distributed according to a Dirichlet distribution with parameter $0.1$, and $K = 100$ rounds. 
        In each round, $30\%$ of devices are selected to participate. In other words, we set $q=0.3$. Each device 
        trains a CNN with two convolutional layers and three fully-connected layers, 
        with model size $W = 5$~MB. We set $\rho = 1{,}000$, $E = 3$, and 
        $B = 20$~MHz. All other parameters are identical to those in the GC10-Det 
        experiment.

        We compare the proposed policy against three baselines: single partition 
        (SP)~\cite{shi2020device}, which is known to be optimal among all 
        single-partition strategies; channel-aware (CA)~\cite{ren2020scheduling}, 
        which allocates bandwidth such that all devices achieve equal transmission 
        rates; and uniform allocation, which assigns equal bandwidth to every device.

		\subsection{Results}

		\subsubsection{Performance Comparision}

		\begin{figure}[!htb]
			\centering
			\includegraphics[width=\columnwidth]{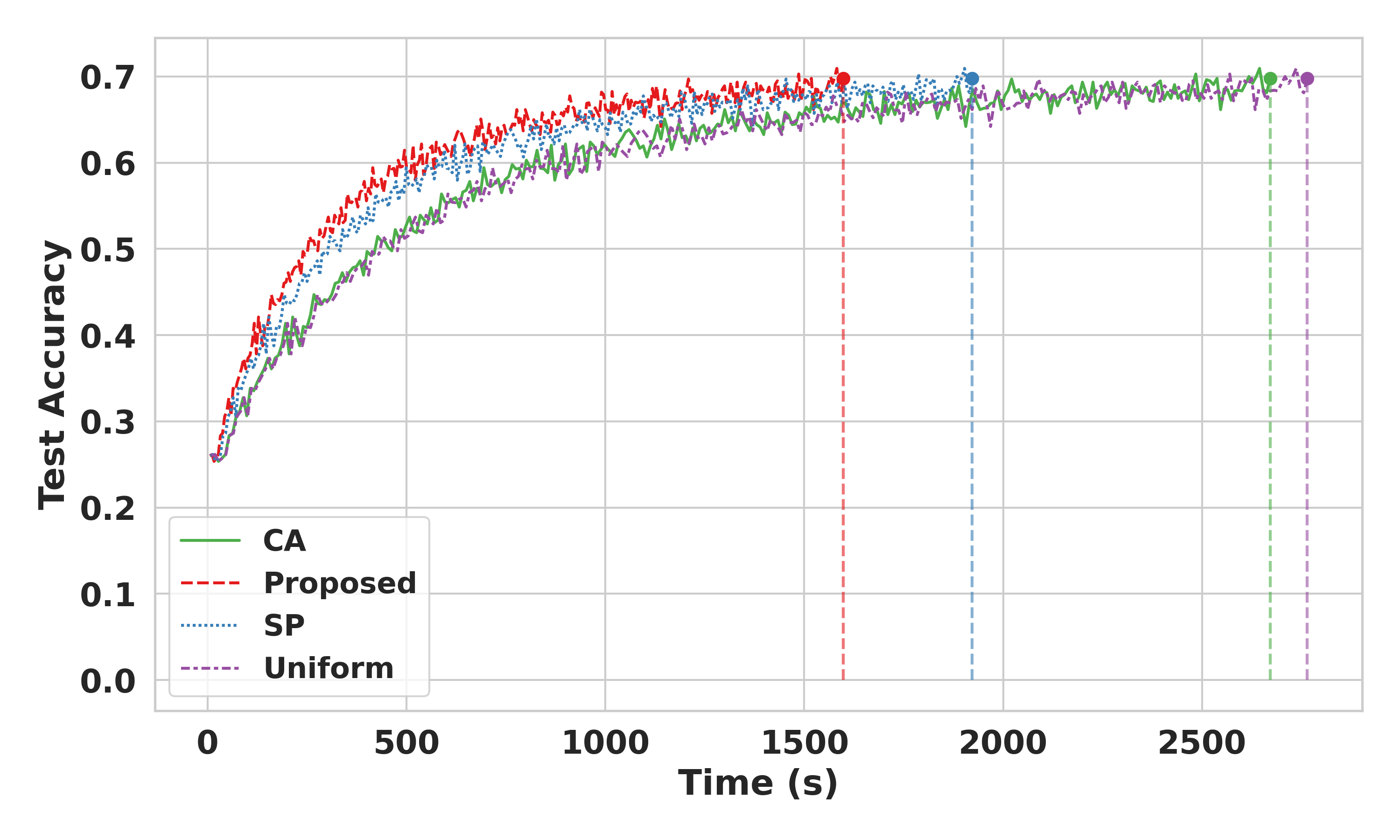}
			\caption{Test accuracy over time for the GC10-Det 
			dataset.}\label{fig:accuracy_vs_time_gc10}
		\end{figure}

		\begin{table}[!htb]
			\centering
			\begin{tabular}{|c|c|c|c|c|}
				\hline
				Method & Proposed & SP & CA & Uniform \\
				\hline
				Training time (s) & 1064.89 & 1281.25 & 1781.31 & 1839.46 \\
				\hline
			\end{tabular}
			\caption{Total training time for the GC10-Det 
			dataset.}\label{tbl:training_time_gc10}
		\end{table}

		\begin{figure}[!htb]
			\centering
			\includegraphics[width=\columnwidth]{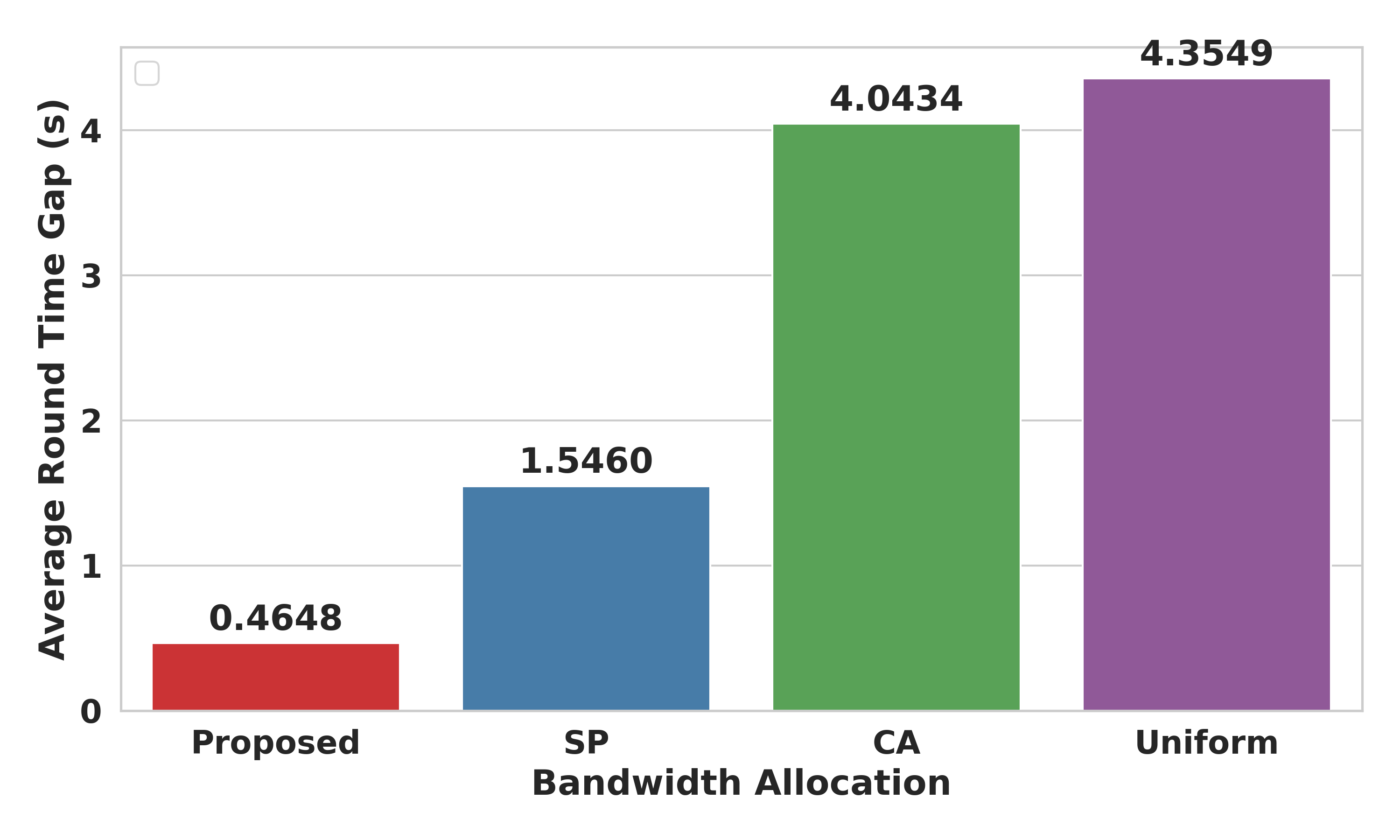}
			\caption{Average round time gap for GC10-Det 
			dataset.}\label{fig:lb_gap_gc10}
		\end{figure}

		\begin{figure}[!htb]
			\centering
			\includegraphics[width=\columnwidth]{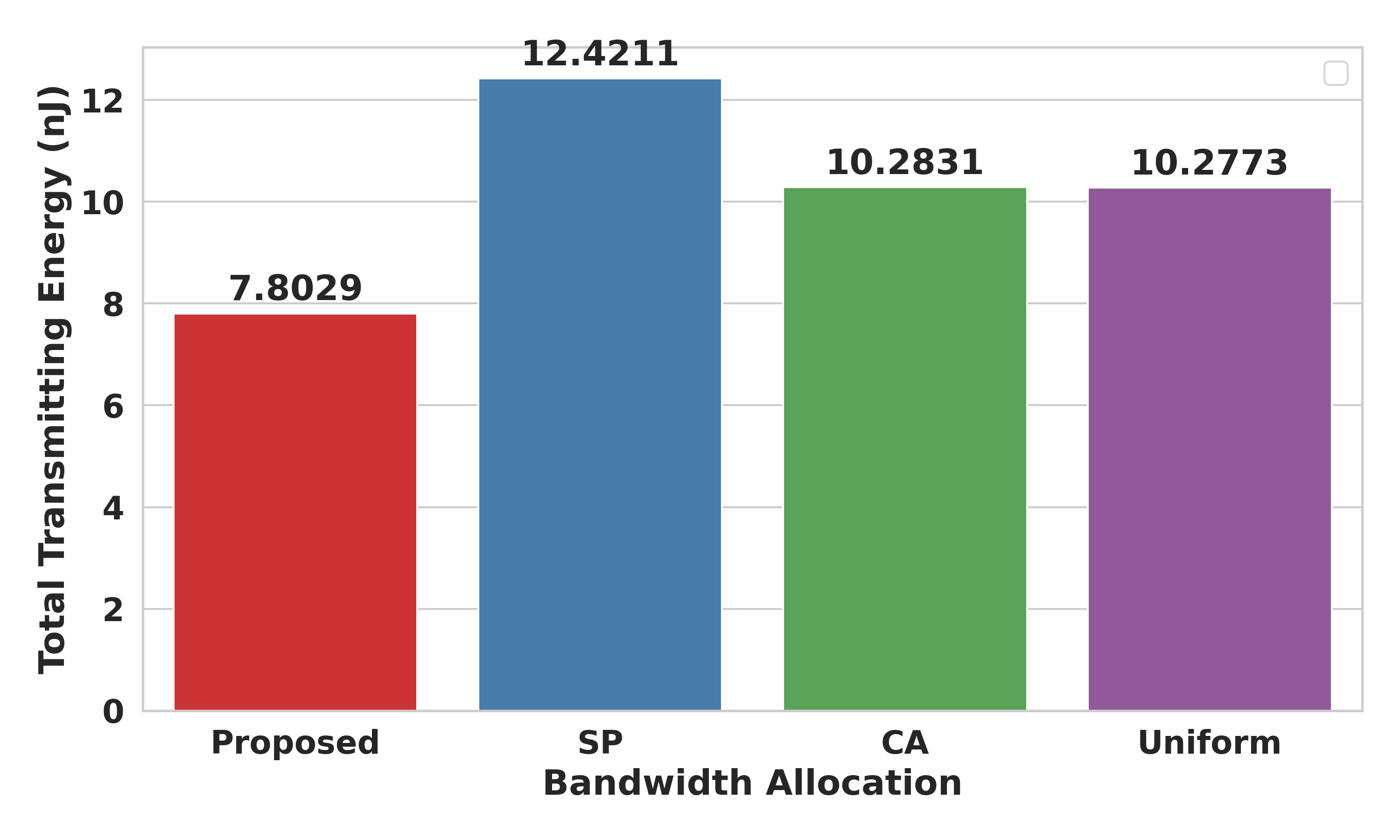}
			\caption{Total uplink energy consumption for GC10-Det 
			dataset.}\label{fig:energy_gc10}
		\end{figure}

		Fig.~\ref{fig:accuracy_vs_time_gc10} and Table~\ref{tbl:training_time_gc10} 
		show the test accuracy and total training time for the GC10-Det dataset. The 
		proposed policy achieves the shortest training time among all methods. 
		Fig.~\ref{fig:lb_gap_gc10} further shows the average round time gap, where 
		the proposed policy most closely approaches the theoretical lower bound. 
		Notably, it yields a substantial reduction over SP, which is itself optimal 
		among single-partition strategies. This demonstrates that device partitioning 
		effectively reduces the communication overhead of FL over IIoT networks.
		
		Fig.~\ref{fig:energy_gc10} reports the total uplink energy consumed by IIoT 
		devices across all rounds. The proposed policy achieves the lowest energy 
		consumption, as allowing devices to transmit with higher bandwidth reduces 
		transmission duration. In contrast, SP exhibits the highest energy consumption 
		despite its second-best round time performance. Under SP, devices that finish 
		computation early must continue transmitting at reduced bandwidth while waiting 
		for slower devices to complete both computation and transmission. This 
		prolonged transmission duration leads to substantially higher energy use.

		\begin{figure}[!htb]
			\centering
			\includegraphics[width=\columnwidth]{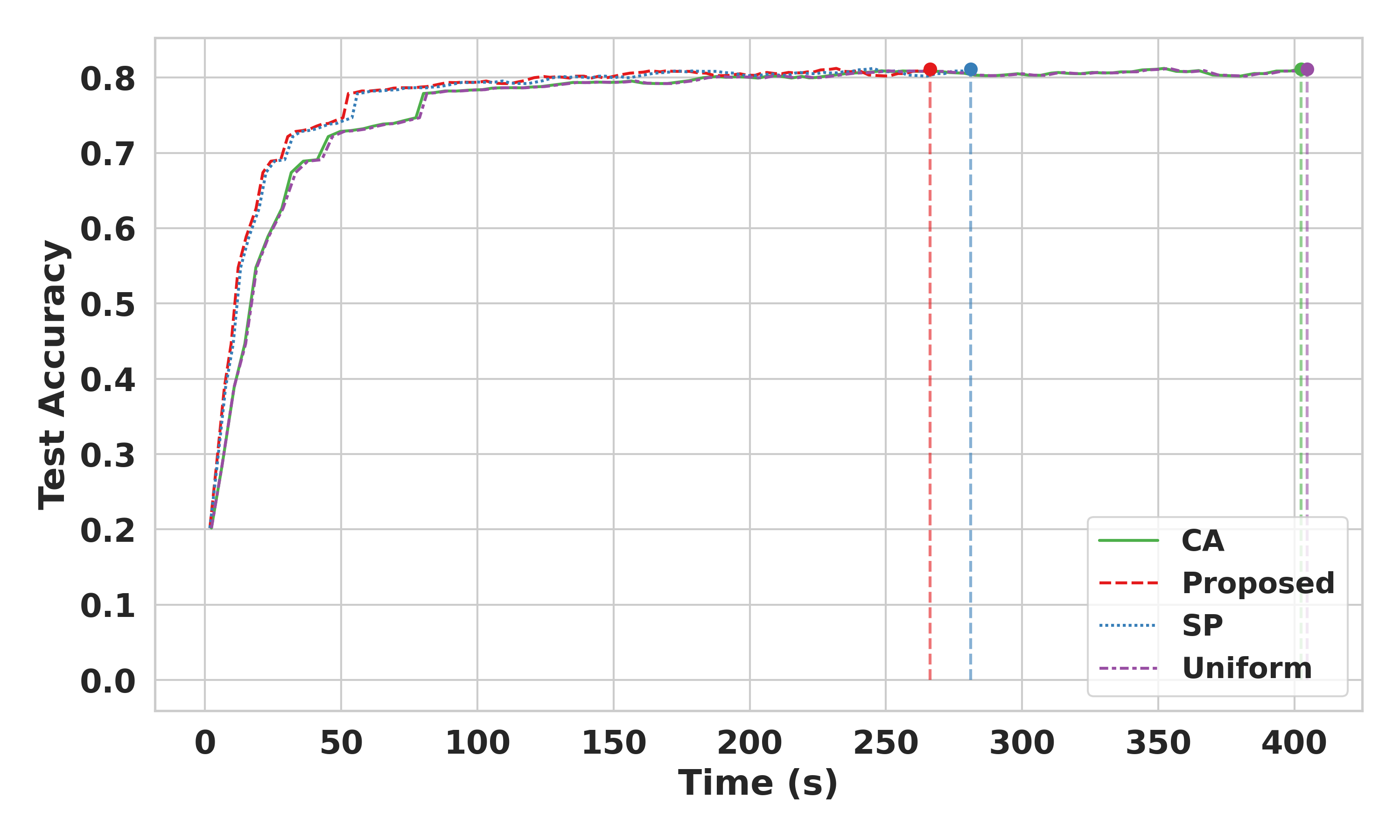}
			\caption{Test accuracy over time for the CIFAR-10 
			dataset.}\label{fig:accuracy_vs_time_cifar10}
		\end{figure}

		\begin{table}[!htb]
			\centering
			\begin{tabular}{|c|c|c|c|c|}
				\hline
				Method & Proposed & SP & CA & Uniform \\
				\hline
				Training time (s) & 133.14 & 140.87 & 201.93 & 203.59 \\
				\hline
			\end{tabular}
			\caption{Total training time for the CIFAR-10 
			dataset.}\label{tbl:training_time_cifar10}
		\end{table}

		\begin{figure}[!htb]
			\centering
			\includegraphics[width=\columnwidth]{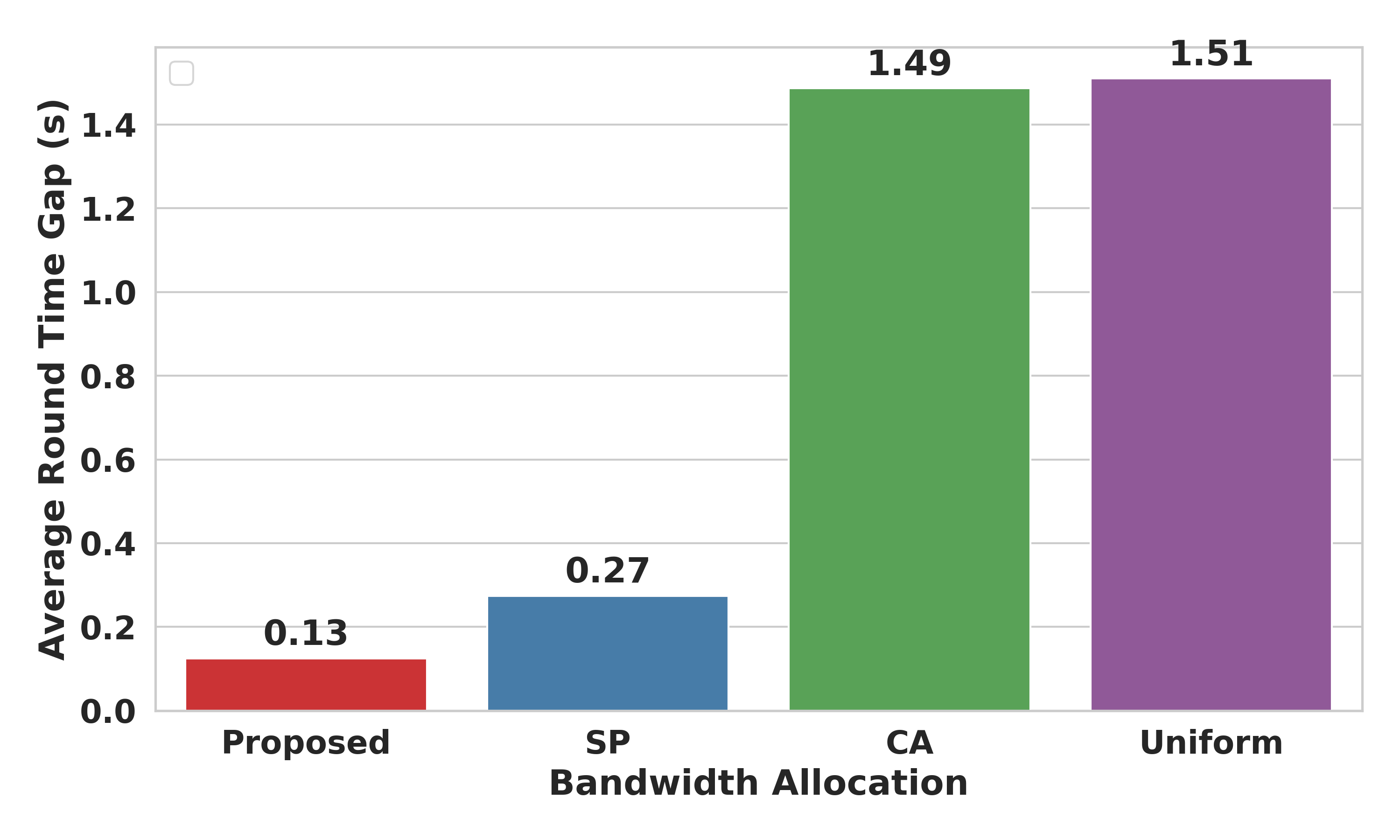}
			\caption{Average round time gap for CIFAR-10 
			dataset.}\label{fig:lb_gap_cifar10}
		\end{figure}

		\begin{figure}[!htb]
			\centering
			\includegraphics[width=\columnwidth]{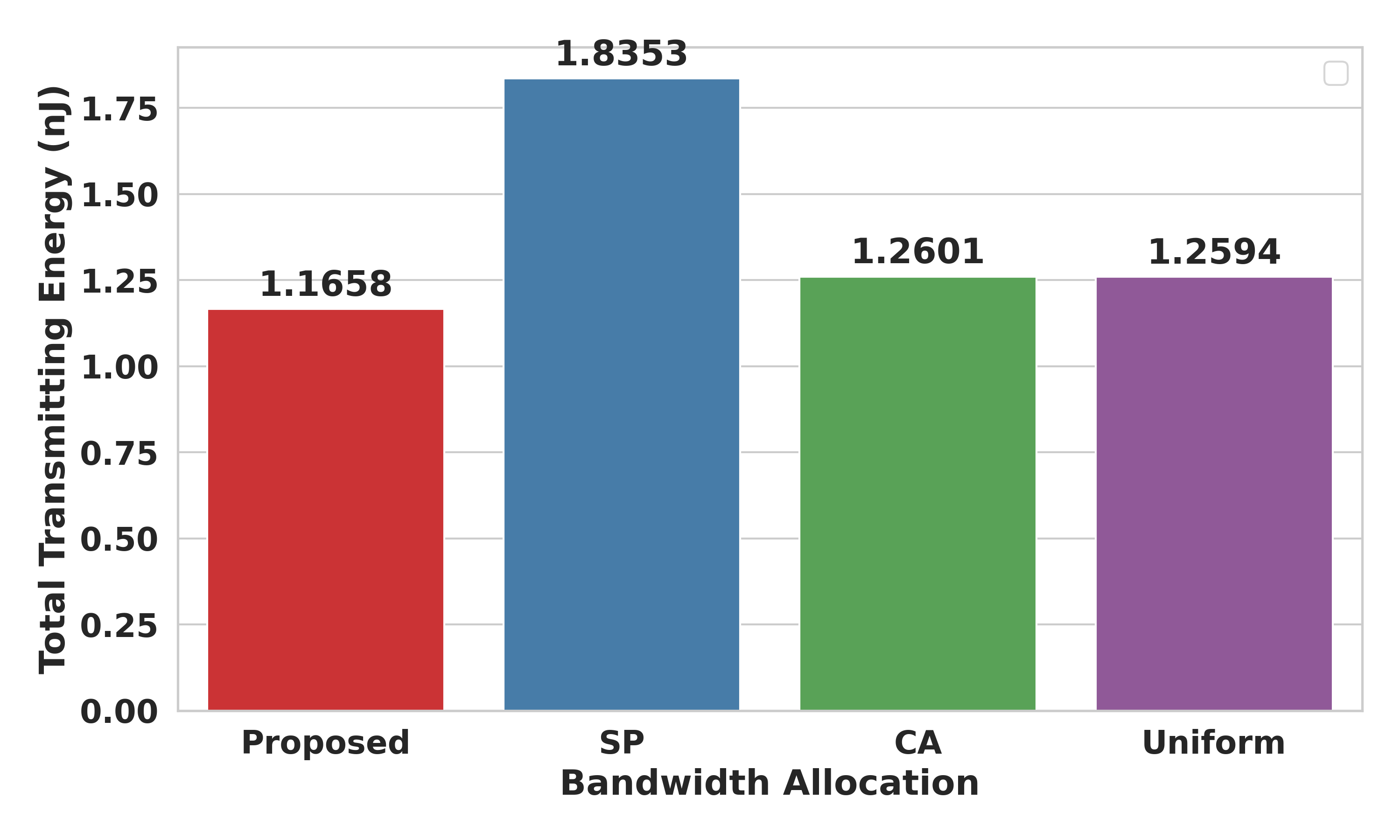}
			\caption{Total uplink energy consumption for CIFAR-10 
			dataset.}\label{fig:energy_cifar10}
		\end{figure}

		Results for CIFAR-10 are shown in Figs.~\ref{fig:accuracy_vs_time_cifar10}, 
		\ref{fig:lb_gap_cifar10}, and Table~\ref{tbl:training_time_cifar10}, 
		exhibiting trends consistent with those observed for GC10-Det. The proposed 
		policy again nearly achieves the lower bound on round time gap. However, the 
		margin over SP is reduced compared to the GC10-Det setting. This is 
		attributable to the smaller model size, which reduces per-round communication 
		time and consequently narrows the absolute deviation of the round time gap 
		from the lower bound across all methods.

		\subsubsection{Impact of Bandwidth and Computation Load}

		\begin{figure}[!htb]
			\centering
			\includegraphics[width=\columnwidth]{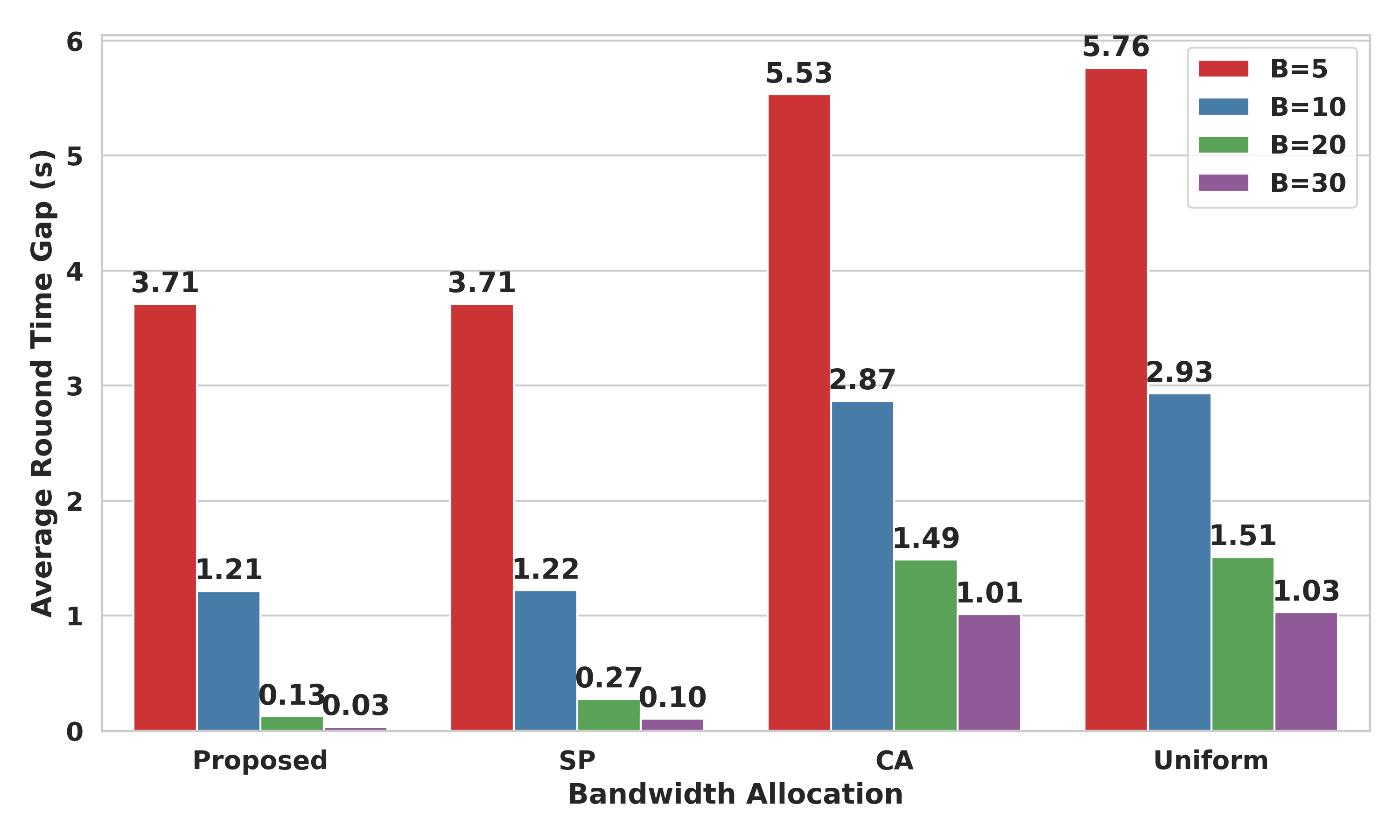}
			\caption{Average round time gap under varying total 
			bandwidth.}\label{fig:bw}
		\end{figure}

		\begin{figure}[!htb]
			\centering
			\includegraphics[width=\columnwidth]{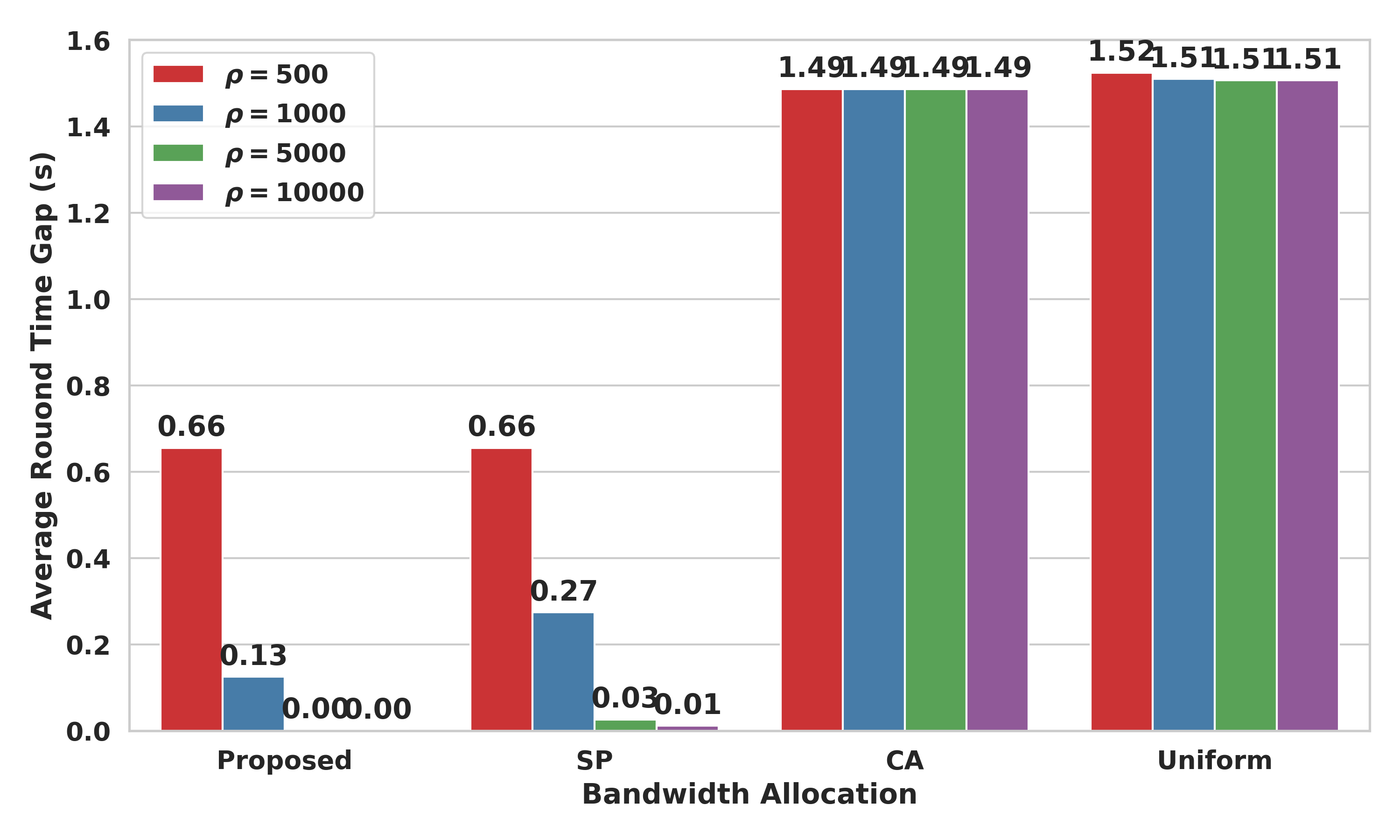}
			\caption{Average round time gap under varying computation 
			load.}\label{fig:compload}
		\end{figure}
		
		Fig.~\ref{fig:bw} shows the average round time gap as a function of total 
		bandwidth $B$, with all other settings identical to the CIFAR-10 experiment. 
		Reducing $B$ increases the round time gap for all policies, as lower bandwidth 
		directly increases per-device transmission time, making it harder to satisfy 
		the partitioning condition~\eqref{ineq:cond_prop2}. This is confirmed at 
		$B = 5$~MHz, where the proposed policy degrades to the performance of SP, 
		indicating that partitioning no longer occurs.

		Conversely, as $B$ increases, the performance gap among policies diminishes. 
		For example, the average round time gaps of the proposed policy and uniform 
		allocation are $0.03$~s and $1.03$~s at $B = 30$~MHz, compared to $1.21$~s 
		and $2.93$~s at $B = 10$~MHz. Under practical noise models, noise power grows 
		proportionally with bandwidth, causing the transmission rate to saturate at 
		large $B$. Consequently, further increasing $B$ yields diminishing reductions 
		in communication time, compressing the differences among policies.

		Fig.~\ref{fig:compload} illustrates the effect of computation load on 
		bandwidth allocation performance. As computation load increases, the round 
		time gaps of the proposed policy and SP decrease, while those of CA and 
		uniform allocation remain largely unchanged. Under heterogeneous computing 
		capabilities, higher computation load amplifies the spread of per-device 
		computation times, which makes the partitioning condition easier to satisfy 
		and allows the proposed policy to approach the lower bound. SP similarly 
		benefits by concentrating bandwidth on the straggler, reducing its 
		transmission time. In contrast, CA and uniform allocation lack mechanisms 
		that adapt to computation heterogeneity, so their performance is insensitive 
		to computation load.

		\begin{figure}[!htb]
			\centering
			\includegraphics[width=1.1\columnwidth]{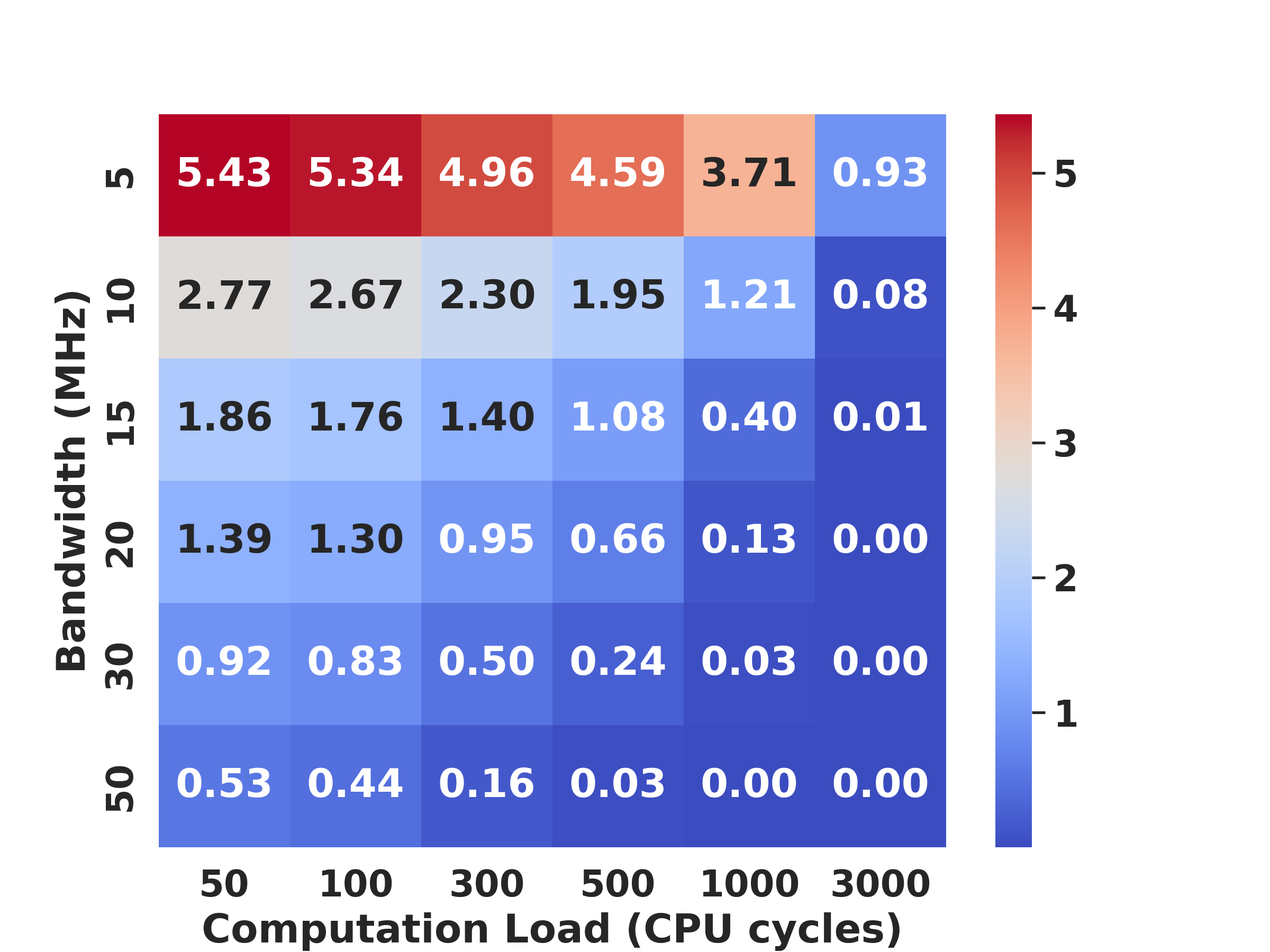}
			\caption{Heatmap of average round time gap across bandwidth and computation 
			load settings.}\label{fig:heat_roundtime}
		\end{figure}

		\begin{figure}[!htb]
			\centering
			\includegraphics[width=1.1\columnwidth]{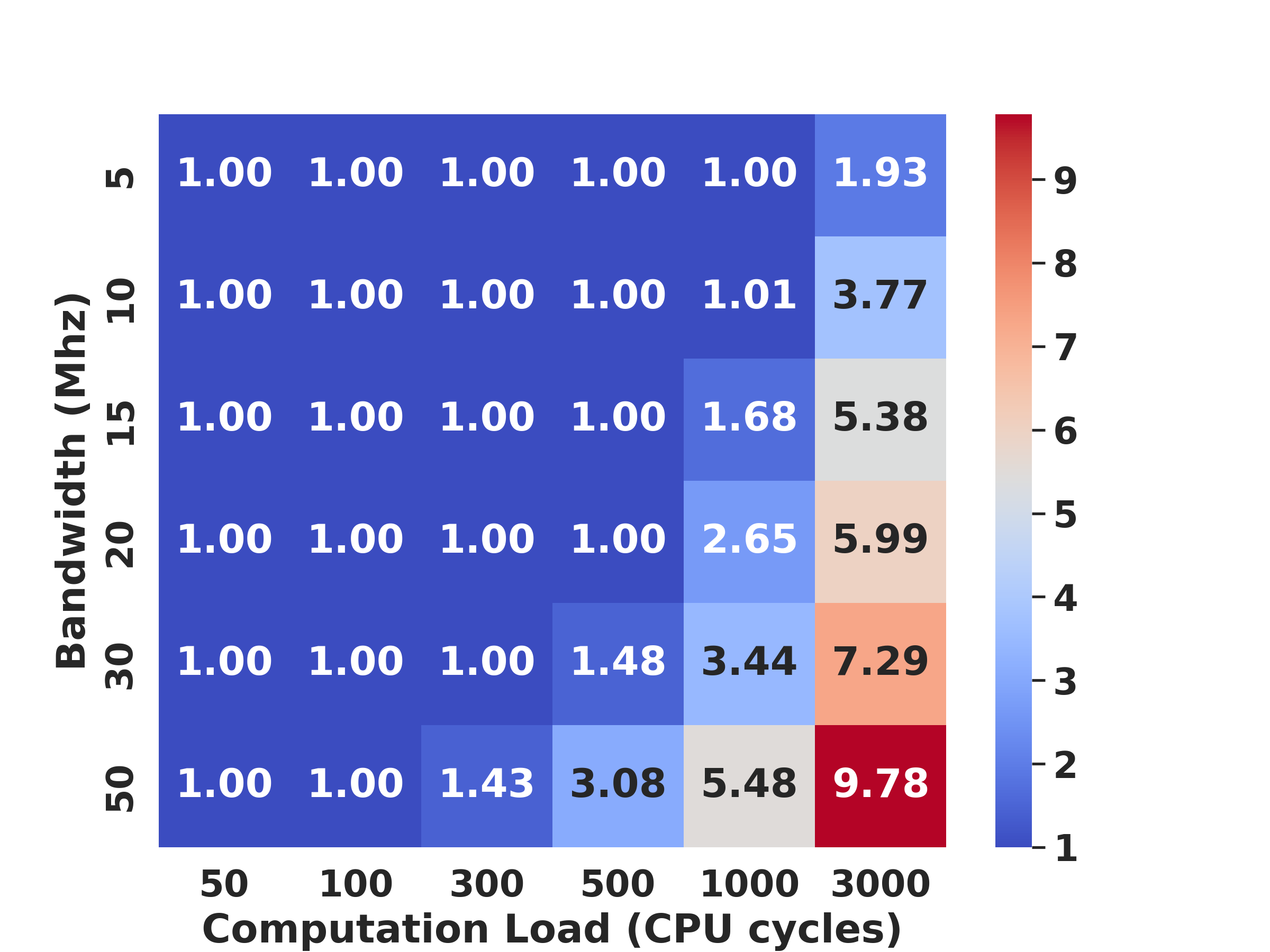}
			\caption{Heatmap of average number of partition groups across bandwidth and 
			computation load settings.}\label{fig:heat_numgroups}
		\end{figure}

		\begin{figure}[!htb]
			\centering
			\includegraphics[width=1.1\columnwidth]{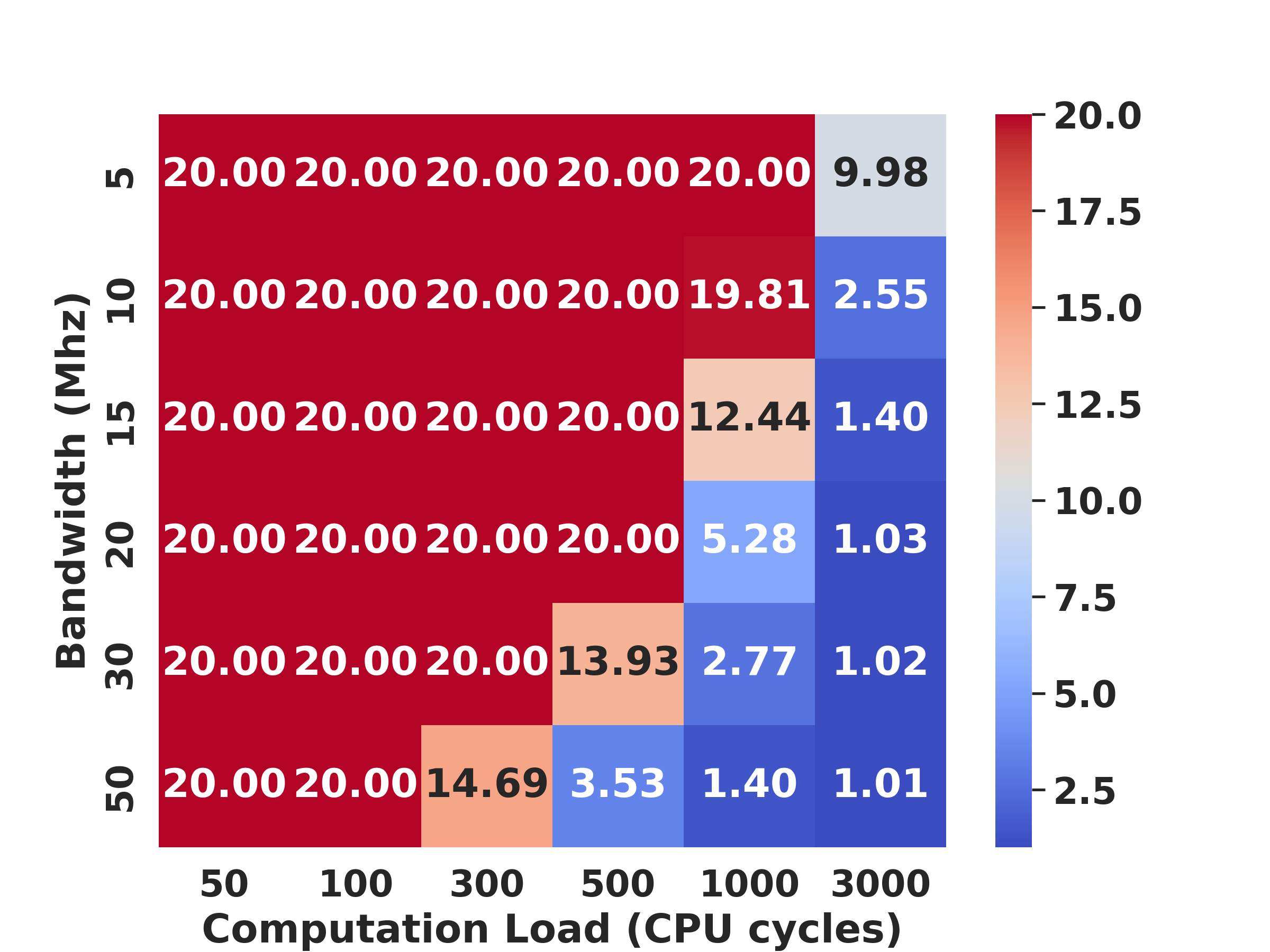}
			\caption{Heatmap of average size of the last group in partition across bandwidth and 
			computation load settings.}\label{fig:heat_lastgroupsize}
		\end{figure}

		Figs.~\ref{fig:heat_roundtime} and~\ref{fig:heat_numgroups} present heatmaps 
		of the average round time gap and average number of partition groups across a 
		range of bandwidth and computation load values. Both increasing $B$ and 
		increasing computation load help the proposed policy approach the lower bound, 
		as larger bandwidth shortens communication time and higher computation load 
		increases the spread of computing times — both of which facilitate partitioning.

		Comparing the two heatmaps, round time gaps within $0.1$~s are achieved only 
		when multiple partitions are formed. However, a large number of groups is not 
		necessary to achieve near-lower-bound performance. For instance, with 
		$\rho = 3{,}000$ fixed, increasing $B$ from $20$ to $50$~MHz raises the 
		average number of groups from $5.99$ to $9.78$, yet the proposed policy 
		already achieves the lower bound at $B = 20$~MHz. This indicates that once 
		the straggler's computation time becomes the bottleneck, the critical factor 
		is ensuring that the straggler receives the full bandwidth $B$ — not the 
		total number of groups. In other words, any policy that grants the final 
		partition exclusive access to the full bandwidth achieves near-lower-bound 
		performance regardless of how many groups precede it. This can be verified in Fig.~\ref{fig:heat_lastgroupsize}, which clearly indicates that the average round time gap is proportional to the size of last group in partition.

		\subsubsection{Impact of Scheduling Algorithm and Number of Devices}

		\begin{figure}[!htb]
			\centering
			\includegraphics[width=\columnwidth]{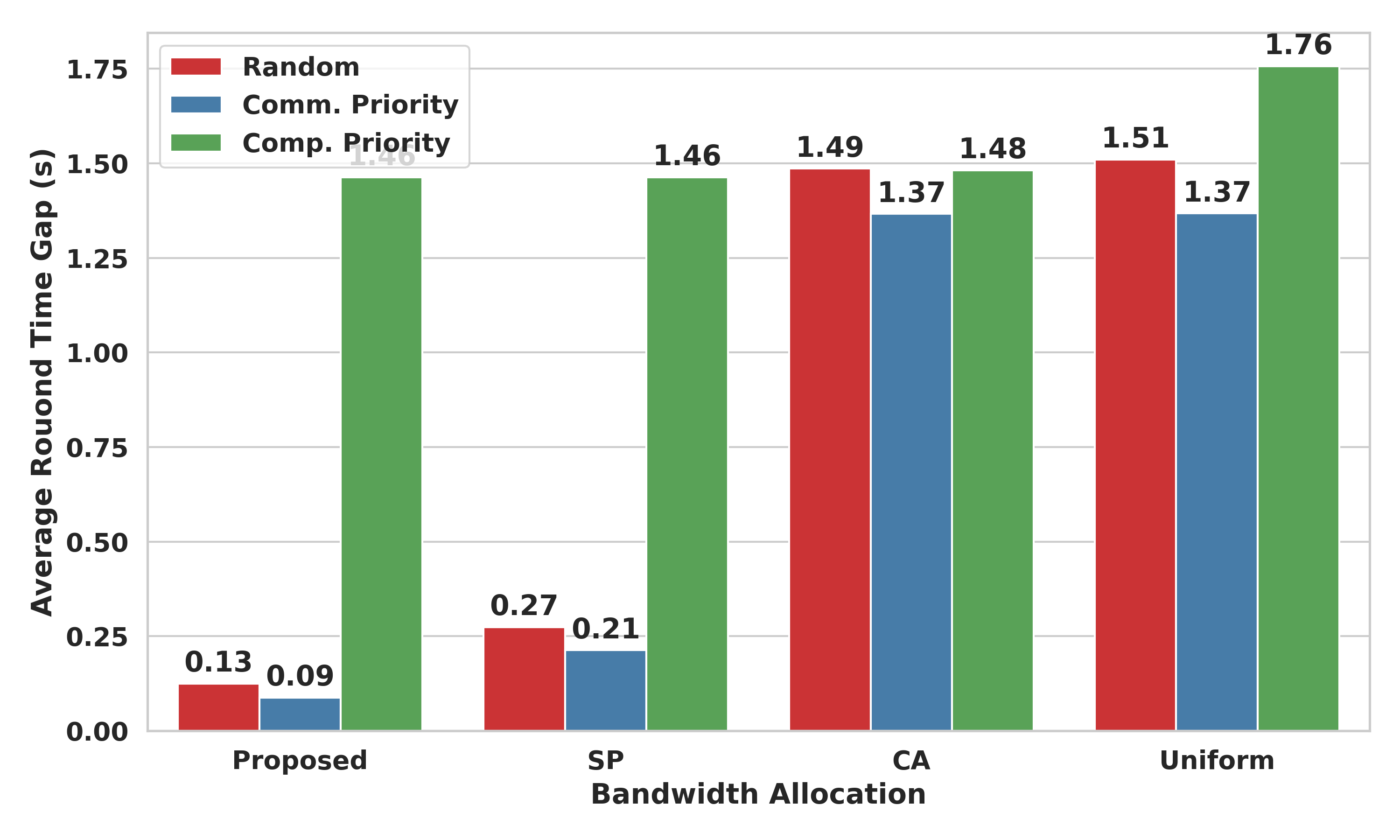}
			\caption{Average round time gap under different scheduling 
			algorithms.}\label{fig:scheduling}
		\end{figure}

		Fig.~\ref{fig:scheduling} compares the performance of bandwidth allocation 
		policies under three scheduling schemes: random, communication-priority, and 
		computation-priority for the same setting of CIFAR10 experiment. Random scheduling selects participating devices with 
		equal probability. Communication-priority scheduling selects devices with the 
		highest channel gains, minimizing communication cost. Computation-priority 
		scheduling selects devices with the shortest local computation times.

		Across all scheduling schemes, the proposed policy achieves the smallest round 
		time gap. However, the advantage over baseline methods narrows under 
		computation-priority scheduling. When devices with similar computation times 
		are selected, the spread of per-device computation times is small, which makes 
		the partitioning condition harder to satisfy and reduces the frequency of 
		partitioning. Additionally, under computation-priority scheduling, CA performs 
		comparably to the proposed and SP policies: since selected devices begin 
		transmission at nearly the same time, ensuring equal transmission rates across 
		devices is sufficient to achieve near-simultaneous completion.

		Among the three scheduling strategies, communication-priority scheduling 
		yields the smallest round time gap for the proposed policy. High channel gains 
		reduce per-device communication time, which allows finer-grained partitioning 
		and reduces the size of the final group. With fewer devices competing for 
		bandwidth in the last partition, each device receives a larger share of $B$, 
		further reducing the straggler's transmission time.

		\begin{figure}[!htb]
			\centering
			\includegraphics[width=\columnwidth]{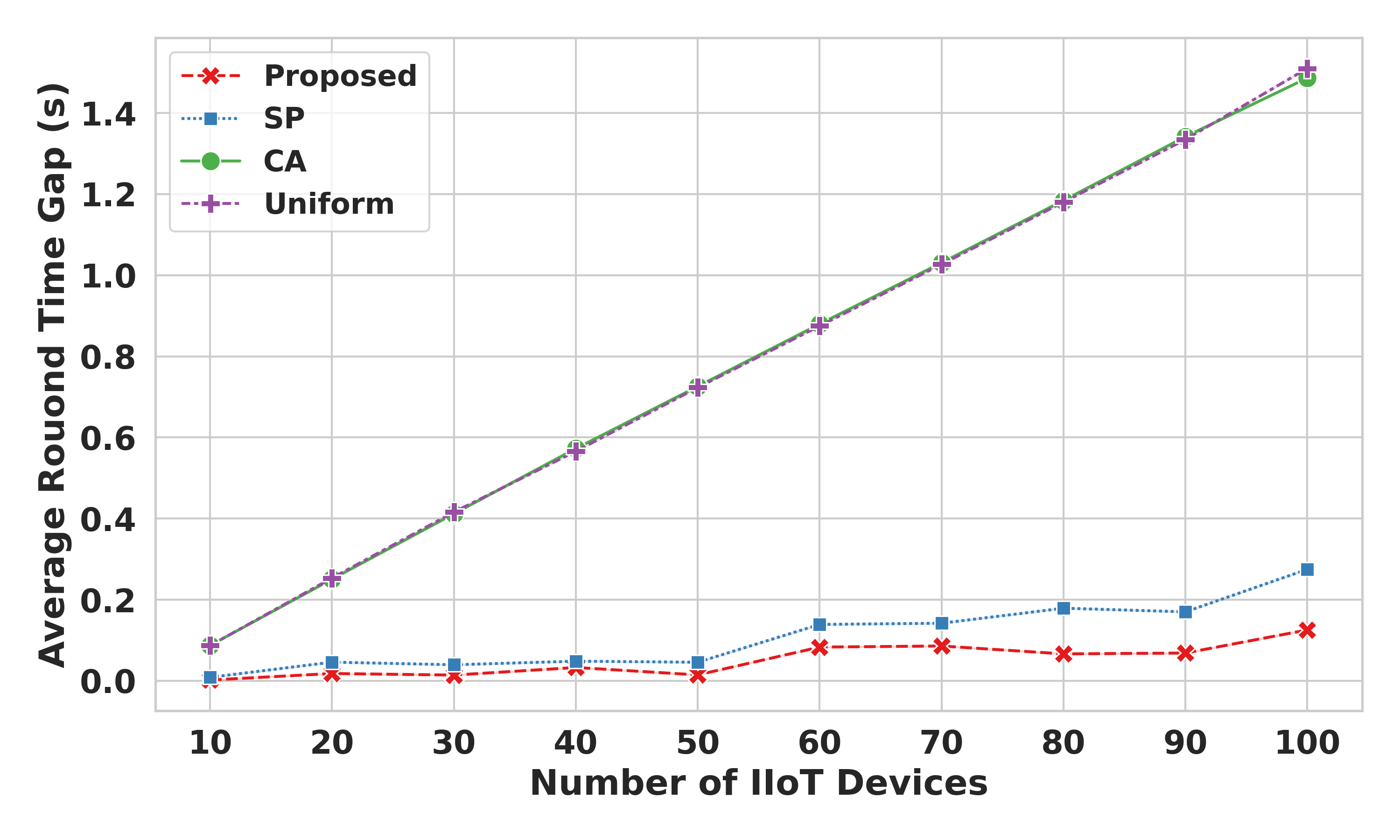}
			\caption{Average round time gap for different number of IIoT devices}\label{fig:num_clients}
		\end{figure}

		Fig.~\ref{fig:num_clients} shows the average round time gap as a function of 
		the number of IIoT devices $N$. As $N$ increases, the performance gap between 
		the proposed policy and the baselines widens. In particular, the round time 
		gaps of CA and uniform allocation grow approximately linearly with $N$, whereas 
		those of the proposed policy and SP increase at a significantly slower rate.

		This divergence is driven by the fact that, given a fixed total bandwidth, the bandwidth available per device diminishes as the network scales. Since the communication round time is dictated by the straggler, the performance of CA and Uniform—which fail to adaptively assign bandwidth to the straggler—degrades as the per-device allocation shrinks. Conversely, our proposed policy and SP dynamically determine bandwidth based on the computation times of the devices, resulting in a sublinear growth of the average round time gap.

		The performance margin between the proposed policy and SP also widens at large 
		$N$. While SP adaptively allocates bandwidth, it is constrained to assign 
		non-zero bandwidth to every selected device simultaneously. At large $N$, this 
		constraint forces the straggler to share bandwidth with many other devices, 
		causing its communication time to notably exceed the theoretical lower bound. 
		The proposed policy alleviates this by grouping devices into sequential 
		subsets, thereby granting the final subset, which contains the straggler exclusive access to the full bandwidth $B$, regardless of $N$.

		Interestingly, the average round time gap of our proposed policy does not follow a strictly monotonic trend. For instance, the gap at $N=40$ is larger than at $N=50$, with similar behavior observed between $N=70$ and $N=80$. This occurs because our policy achieves a smaller round time gap when the size of the final group in the partition is minimized. Since the final group size is not a monotone function of $N$, adding a device 
		with a sufficiently long computation time can cause it to be assigned to its 
		own group, effectively reducing the final group size and thereby decreasing 
		the round time gap.

		\section{Conclusion}\label{sec:conclusion}
    	In this paper, we investigated bandwidth allocation for FL over IIoT networks 
		and proposed DBBP, a bandwidth allocation policy. In each 
		communication round, DPBP partitions the selected IIoT devices into multiple 
		ordered subsets and grants each subset exclusive access to the full bandwidth 
		sequentially. We formally proved that this partitioning approach achieves a 
		lower round time gap than the optimal single-partition strategy, irrespective 
		of the underlying scheduling algorithm. Extensive experiments on the GC10-Det 
		and CIFAR-10 datasets confirmed that DPBP consistently reduces total training 
		time and uplink energy consumption compared to existing bandwidth allocation 
		methods, validating the practical benefit of device partitioning in FL systems 
		deployed over IIoT networks.

		Several directions remain open for future work. First, as energy efficiency is 
		a critical constraint in IIoT deployments, extending DPBP to incorporate 
		explicit energy consumption objectives is a natural next step. Second, joint optimization of device scheduling and bandwidth partitioning presents a promising avenue for further reducing learning time, as the current work treats scheduling as an independent component. Scheduling algorithm to improve the performance of DBBP further can be optimized. Finally, while DPBP is designed for IIoT settings, its applicability 
		extends to any domain where edge devices compete for limited bandwidth 
		resources, including vehicular networks, mobile federated learning, and 
		smart manufacturing.
				
	\bibliographystyle{IEEEtran}
	\bibliography{references}

	\end{document}